%% file: main.tex
\documentclass[10pt,twocolumn,letterpaper]{article}

\usepackage{cvpr}
\input{style/packages.tex}
\graphicspath{{figures/}}

\cvprfinalcopy

\def\cvprPaperID{853} 

\ifcvprfinal\fi
\begin{document}

\def\mytitle{End-to-End Training of Hybrid CNN-CRF Models for Stereo}
\title{\mytitle}

\author{Patrick Kn\"obelreiter$^1$\\
{\tt \small knoebelreiter@icg.tugraz.at}\\
\and
Christian Reinbacher$^1$\\
{\tt \small reinbacher@icg.tugraz.at}\\
\and
Alexander Shekhovtsov$^1$\\
{\tt \small shekhovtsov@icg.tugraz.at}\\
\and
Thomas Pock$^{1,2}$\\
{\tt \small pock@icg.tugraz.at}\\ 
\and
$^1$Institute for Computer Graphics and Vision\\
Graz University of Technology
\and
$^2$Center for Vision, Automation \& Control\\
AIT Austrian Institute of Technology \\
}

\maketitle
\thispagestyle{empty}

\newboolean{InAppendix}
\setboolean{InAppendix}{false}

\begin{abstract}
We propose a novel and principled hybrid CNN+CRF model for stereo estimation. Our model allows to exploit the advantages of both, convolutional neural networks (CNNs) and conditional random fields (CRFs) in an unified approach. The CNNs compute expressive features for matching and distinctive color edges, which in turn are used to compute the unary and binary costs of the CRF. For inference, we apply a recently proposed highly parallel dual block descent algorithm which only needs a small fixed number of iterations to compute a high-quality approximate minimizer. As the main contribution of the paper, we propose a theoretically sound method based on the structured output support vector machine (SSVM) to train the hybrid CNN+CRF model on large-scale data end-to-end. Our trained models perform very well despite the fact that we are using shallow CNNs and do not apply any kind of post-processing to the final output of the CRF. We evaluate our combined models on challenging stereo benchmarks such as Middlebury 2014 and Kitti 2015 and also investigate the performance of each individual component.

\end{abstract}

\input{tex/intro.tex}
\input{tex/related_work.tex}
\input{tex/model.tex}
\input{tex/training.tex}
\input{tex/implementation_details.tex}
\input{tex/experiments.tex}

\input{tex/conclusion.tex}
\subsection*{Acknowledgements}
  \small
  This work was supported by the research initiative Intelligent Vision Austria with funding
  from the AIT and the Austrian Federal Ministry of Science, Research and Economy
  HRSM programme (BGBl. II Nr. 292/2012) and the ERC starting grant HOMOVIS, No. 640156.

{\small
\bibliographystyle{apa}
\bibliography{bib/strings,bib/main,bib/cnn-crf,bib/ssvm,bib/suppl,bib/dmm}
}

\clearpage
\newpage
\appendix
\setboolean{InAppendix}{true}
\input{tex/supplementary.tex}

\end{document}

%% file: style/packages.tex
\usepackage{etex}
\usepackage{url}
\usepackage[numbers]{natbib}
\usepackage{epsfig}
\usepackage{graphicx}
\usepackage{array}
\usepackage{tabularx}
\usepackage{graphicx}
\usepackage{amsmath,amssymb,amsthm,amsxtra}
\makeatletter
\let\proof\@undefined                        
\let\endproof\@undefined                  
\makeatother
\usepackage{stmaryrd}
\usepackage{comment}
\usepackage[ruled,vlined,linesnumbered,procnumbered,longend]{algorithm2e}

\SetCommentSty{mycommfont}
\usepackage{multirow,rotating}
\usepackage{array}
\usepackage{afterpage}
\usepackage{dblfloatfix}
\usepackage{setspace}
\usepackage{booktabs}
\usepackage{color, colortbl}
\usepackage{tabularx,calc}
\usepackage{tikz}                 
\usetikzlibrary{patterns,fit,external,arrows,positioning,calc,decorations}
\usepackage{pgfplots}             
\usepackage{tikzscale}
\tikzset{
    >=stealth',
    punkt/.style={
           rectangle,
           rounded corners,
           draw=black, very thick,
           minimum height=2em,
           text centered},
    pil/.style={
           ->,
           thick,
           shorten <=2pt,
           shorten >=2pt,}
}
\usepackage{overpic}
\usepackage{bbm}
\usepackage[absolute]{textpos}
\usepackage{units}
\usepackage{xifthen}
\usepackage{adjustbox}

\usepackage[page,header]{appendix}
\usepackage{titletoc}

\usepackage{times}

\usepackage{marginnote}
\usepackage[textwidth=2cm,figwidth=\columnwidth,colorinlistoftodos]{todonotes}
\makeatletter
\renewcommand{\todo}[2][]{\tikzexternaldisable\@todo[#1]{#2}\tikzexternalenable}
\makeatother

\newcounter{mycomment} 

\usepackage{tabularx}
\usepackage{array}
\usepackage[position=bottom,font=small]{caption}
\usepackage{longtable}
\usepackage{tabu} 
\usepackage{rotating}
\newcommand{\sw}[1]{\begin{sideways}#1\end{sideways}}
\usepackage{dcolumn}
\newcolumntype{d}[1]{D{.}{.}{#1} }
\usepackage{anyfontsize}
\usepackage[most]{tcolorbox}
\usepackage{pbox}
\newlength{\luw}
\newlength{\luh}

\usepackage{chngcntr}

\usepackage[draft=false, pagebackref=false,breaklinks=true,colorlinks,bookmarks=false,pdftex]{hyperref}
\input{style/tex_defs.tex}
\input{style/myspace.tex}
\usepackage{thmtools, thm-restate}
\usepackage[capitalise,nameinlink]{cleveref}
\crefformat{equation}{(#2#1#3)}
\crefname{section}{\parSym}{\parSym\parSym}
\Crefname{section}{\parSym}{\parSym\parSym}

\crefname{appendix}{Suppl.}{Suppl.}
\Crefname{appendix}{Suppl.}{Suppl.}

%% file: style/tex_defs.tex
\DeclareMathAlphabet{\mathcalligra}{T1}{calligra}{m}{n}
\DeclareMathAlphabet{\mathantt}{OT1}{antt}{li}{it}
\DeclareMathAlphabet{\mathpzc}{OT1}{pzc}{m}{it}

\newcommand{\argmax}{\mathop{\rm argmax}}

\newcommand{\argmin}{\mathop{\rm argmin}}

\DeclareMathOperator{\dom}{dom}

\newcommand{\<}{\langle}
\renewcommand{\>}{\rangle}

\renewcommand{\mid}{\:|\,}

\providecommand{\e}[1]{\ensuremath{\times 10^{#1}}}

\newcommand{\Real}{\mathbb{R}}

\newcommand{\tab}{{\hphantom{bla}}}

\newcommand{\V}{\mathcal{V}}

\newcommand{\E}{\mathcal{E}}

\newcommand{\X}{\mathcal{X}}

\renewcommand{\L}{\mathcal{L}}

\newcounter{myRomanCounter}

\newcommand{\w}{{w}}

\newcommand{\gray}{\color[rgb]{0.5,0.5,0.5}}
\newcommand{\red}{\color[rgb]{1,0,0}}

\renewcommand*{\paragraph}[1]{\par\noindent{\normalsize\bf #1}\,\xspace}

\def\mathrlap{\mathpalette\mathrlapinternal}
\def\mathllap{\mathpalette\mathllapinternal}
\def\mathllapinternal#1#2{\llap{$\mathsurround=0pt#1{#2}$}}
\def\mathrlapinternal#1#2{\rlap{$\mathsurround=0pt#1{#2}$}}

\def\leftbb{\mathrlap{[}\hskip1.3pt[}
\def\rightbb{]\hskip1.36pt\mathllap{]}}

\def\epsilon{\varepsilon}

\makeatletter
\let\parSym\S
\def\S{\mathcal{S}}
\makeatother


\newcommand{\revisit}[1][]{%
\ifthenelse{\equal{#1}{}}{
\ensuremath{\red \triangle}\xspace}{%
{\ensuremath{\red \rhd}\xspace}%
{\gray #1}%
{\ensuremath{\red \lhd}\xspace}%
}%
}

\def\anchor [#1]#2{%
\phantomsection{}#1\label{#2}%
\def\arga{#2}%
\global\expandafter\def\csname#2\endcsname{%
\hyperref[#2]{#1}\xspace%
}%
}%

\def\codefunction [#1]#2{%
\phantomsection{}\label{#2}{\ttfamily #1\xspace}%
\def\arga{#2}%
\global\expandafter\def\csname#2\endcsname{%
\hyperref[#2]{\ttfamily #1}\xspace%
}%
}

\usepackage{array}
\newcolumntype{L}[1]{>{\raggedright\let\newline\\\arraybackslash\hspace{0pt}}m{#1}}
\newcolumntype{C}[1]{>{\centering\let\newline\\\arraybackslash\hspace{0pt}}m{#1}}
\newcolumntype{R}[1]{>{\raggedleft\let\newline\\\arraybackslash\hspace{0pt}}m{#1}}

\SetKwFor{forever}{while true}{}{end while}

\def\epsilon{\varepsilon}

\newcommand{\unarycnn}{{\em Unary-CNN}\xspace}
\newcommand{\pairwisecnn}{{\em Pairwise-CNN}\xspace}

%% file: style/myspace.tex
\newlength{\myskip}
  {\begin{list}{\arabic{enumi}.}%
     {\topsep=0in\itemsep=0in\parsep=0pt\partopsep=0in\usecounter{enumi}}%
   }{\end{list}}

\renewenvironment{itemize}%
  {\begin{list}{$\bullet$}%
     {\topsep=0in\itemsep=0pt\parsep=0pt\partopsep=0in\usecounter{itemi}}%
   }{\end{list}\addvspace{0pt}}

\SetAlgoSkip{skipalgomyskip}
\SetAlgoInsideSkip{}

\makeatletter
\let\corollary\@undefined
\let\endcorollary\@undefined
\let\definition\@undefined
\let\enddefinition\@undefined
\let\proof\@undefined
\let\endproof\@undefined
\let\theorem\@undefined
\let\c@theorem\@undefined
\let\endtheorem\@undefined
\let\lemma\@undefined
\let\endlemma\@undefined
\let\example\@undefined
\let\c@example\@undefined
\let\endexample\@undefined
\let\remark\@undefined
\let\endremark\@undefined
\let\proposition\@undefined
\let\endproposition\@undefined
\let\property\@undefined
\let\endproperty\@undefined
\makeatother

\usepackage{thmtools}

\newtheoremstyle{tightItalic}
  {0.5\myskip}
  {0\myskip}
  {}
  {}
  {\itshape}
  {.}
  { }
  {}

\newtheoremstyle{tightBf}
  {0.5\myskip}
  {0\myskip}
  {}
  {}
  {\bf}
  {.}
  {.5em}
  {}

\theoremstyle{definition}
\theoremstyle{tightBf}
\declaretheorem[style=tightBf,parent=section]{thm}
\numberwithin{thm}{section}
\declaretheorem[style=tightBf,sibling=thm,name=Theorem]{theorem}
\declaretheorem[style=tightBf,sibling=theorem,name=Lemma]{lemma}
\declaretheorem[style=tightBf,sibling=theorem,name=Definition]{definition}
\declaretheorem[style=tightBf,sibling=theorem,name=Corollary]{corollary}
\declaretheorem[style=tightBf,sibling=theorem,name=Proposition]{proposition}

\declaretheorem[style=tightBf,parent=section,name=Example]{example}
\numberwithin{example}{section}

\theoremstyle{tightItalic}
\newtheorem*{proof}{Proof}

%% file: tex/intro.tex

\section{Introduction}
Stereo matching is a fundamental low-level vision problem. It is an ill-posed inverse problem, asking to reconstruct the depth from a pair of images. This requires robustness to all kinds of visual nuisances as well as a good prior model of the 3D environment. 
Prior to deep neural network data-driven approaches, progress had been made using global optimization techniques~\cite{Kolmogorov2006,Laude2016,ranftl2014non,Scharstein2002,Woodford-08} featuring robust surface models and occlusion mechanisms. Typically, these methods had to rely on engineered cost matching and involved choosing a number of parameters experimentally.
\par
Recent deep CNN models for stereo~\cite{Chen2015,Luo2016,Zbontar2016} learn from data to be robust to illumination changes, occlusions, reflections, noise, \etc. A deep and possibly multi-scale architecture is used to leverage the local matching to a global one. 
However, also deep CNN models for stereo rely a lot on post-processing, combining a set of filters and optimization-like heuristics, to produce final accurate results.
\par
In this work we combine CNNs with a discrete optimization model for stereo. This allows complex local matching costs and parametrized geometric priors to be put together in a global optimization approach and to be learned end-to-end from the data. 
Even though our model contains CNNs, it is still easily interpretable. This property allows us to shed more light on the learning our network performs. We start from a CRF formulation and replace all hand-crafted terms with learned ones. 

\begin{figure}[t]
\centering
\begin{tikzpicture}[node distance=0.6cm, auto,]
 \node[punkt] (i0) {$I_0$};
 \node[punkt,below=0.5cm of i0] (i1) {$I_1$};
 \node[punkt,right=of i0,fill=blue!20] (unary0) {Unary CNN};
 \node[punkt,right=of i1,fill=blue!20] (unary1) {Unary CNN};
 \node[punkt,right= 3.5cm of  {$(i0)!0.5!(i1)$},fill=green!20] (correlation) {Correlation};
 \node[punkt,right=of correlation,fill=magenta!20] (crf) {CRF};
 \node[punkt,above= 1cm of {$(unary0)!0.6!(correlation)$},fill=orange!30] (pw) {Contrast Sensitive / Pairwise CNN};
 \node[punkt,right=of crf] (disp) {D};

 \draw[pil,->](i0) -- (unary0);
 \draw[pil,->](i1) -- (unary1);
 \draw[pil,->](unary0.east) -- (correlation.west);
 \draw[pil,->](unary1.east) -- (correlation.west);
 \draw[pil,->](correlation) -- (crf);
 \draw[pil,->](i0.east) to [bend left=25] (pw.west);
 \draw[pil,->](pw.east) to [bend left=25] (crf.north);
 \draw[pil,->](crf) -- (disp);
\end{tikzpicture}
\caption{Architecture: A convolutional neural network, which we call \unarycnn computes features of the two images for each pixel. The features are compared using a {\em Correlation} layer. The resulting matching cost volume becomes the unary cost of the {\em CRF}. The pairwise costs of the CRF are parametrized by edge weights, which can either follow a usual contrast sensitive model or estimated by the \pairwisecnn.
%
}
\label{fig:model}
\end{figure}
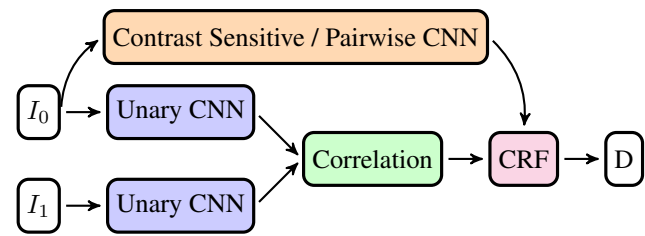
We propose a hybrid CNN-CRF model illustrated in~\cref{fig:model}. Our \unarycnn computes local features of both images which are then compared in a fixed correlation metric. Our \pairwisecnn can additionally estimate contrast-sensitive pairwise costs in order to encourage or discourage label jumps. Using the learned unary and pairwise costs, the CRF tries to find a joint solution optimizing the total sum of all unary and pairwise costs in a 4-connected graph. This model generalizes existing engineered approaches in stereo as well as augment existing fully learned ones. The \unarycnn straightforwardly generalizes manually designed matching costs such as those based on differences of colors, sampling-insensitive variants~\cite{Birchfield-98}, local binary patterns (\eg, Census transform~\cite{Zabih1994}), \etc. The \pairwisecnn generalizes a contrast-sensitive regularizer~\cite{Boykov01b}, which is the best practice in {MRF/CRF} models for segmentation and stereo.

\par
To perform inference in the CRF model we apply the fast method of~\cite{Discrete-Continuous-16}, which improves over heuristic approaches combining multiple post-processing steps as used in~\cite{Chen2015,Luo2016,Zbontar2016}.
We deliberately chose not to use any post-processing in order to show that most of the performance gain through post-processing can be covered by a well-trained CRF model. 
%
While previously, methods based on LP-relaxation were considered prohibitively expensive for stereo,~\cite{Discrete-Continuous-16} reports a near real-time performance, which makes this choice definitely faster than a full deep architecture~\cite{Zbontar2016} and competitive in speed with inference heuristics such as SGM~\cite{Hirschmueller2005}, MGM~\cite{Facciolo-15}, \etc.
\par
We can train the complete model shown in~\cref{fig:model} using the structured support vector machine (SSVM) formulation and propagating its subgradient through the networks. 
Training a non-linear CNN+CRF model of this scale is a challenging problem that has not been addressed before.
We show this is practically feasible by having a fast inference method and using an approximate subgradient scheme. 
Since at test time the inference is applied to complete images, we train it on complete images as well.
This is in contrast to the works~\cite{Luo2016,Zagoruyko-15,Zbontar2016} which sample patches for training. The SSVM approach optimizes the inference performance on complete images of the training set more directly. While with the maximum likelihood it is important to sample hard negative examples (hard mining)~\cite{SimoSerraICCV2015}, the SSVM determines labellings that are hard to separate as the most violated constraints.
\par
We observed that the hybrid CNN+CRF network performs very well already with shallow CNN models, such as 3-7 layers. With the CRF layer the generalization gap is much smaller (less overfitting) than without. Therefore a hybrid model can achieve a competitive performance using much fewer parameters than the state of the art. This leads to a more compact model and a better utilization of the training data.
\par
We report competitive performance on benchmarks using a shallow hybrid model.
 Qualitative results demonstrate that our model is often able to delineate object boundaries accurately and it is also often robust to occlusions, although our CRF did not include explicit occlusion modeling.
\paragraph{Contribution}
We propose a hybrid CNN+CRF model for stereo, which utilizes the expressiveness of CNNs to compute good unary- as well as pairwise-costs and uses the CRF to easily integrate long-range interactions.
We propose an efficient approach to train our CNN+CRF model.
The trained hybrid model is shown to be fast and yields competitive results on challenging datasets.
We do not use any kind of post-processing.
The code to reproduce the results will be made publicly available\footnote{\url{https://github.com/VLOGroup}}. 

%% file: tex/related_work.tex
\section{Related Work}
\paragraph{CNNs for Stereo}
Most related to our work are CNN matching networks for stereo proposed by~\cite{Chen2015,Luo2016} and the {\em fast} version of \cite{Zbontar2016}. They use similar architectures with a siamese network~\cite{Bromley1994} performing feature extraction from both images and matching them using a fixed correlation function (product layer). Parts of our model~(see \cref{fig:model}) denoted as \unarycnn and {\em Correlation} closely follow these works.
However, while~\cite{Chen2015,Luo2016, Zbontar2016} train by sampling matching and non-matching image patches, following the line of work on more general matching / image retrieval, we train from complete images. Only in this setting it is possible to extend to a full end-to-end training of a model that includes a CRF (or any other global post-processing) optimizing specifically for the best performance in the dense matching. The {\em accurate} model of \cite{Zbontar2016} implements the comparison of features by a fully connected NN, which is more accurate than their {\em fast} model but significantly slower. All these methods make an extensive use of post-processing steps that are not jointly-trainable with the CNN: \cite{Zbontar2016} applies cost cross aggregation, semi-global matching, subpixel enhancement, median and bilateral filtering; \cite{Luo2016} uses window-based cost aggregation, semi-global matching, left-right consistency check, subpixel refinement, median filtering, bilateral filtering and slanted plane fitting; \cite{Chen2015} uses semi-global matching, left-right consistency check, disparity propagation and median-filtering.
Experiments in~\cite{Luo2016} comparing bare networks without post-processing show that their fixed correlation network outperforms the {\em accurate} version of~\cite{Zbontar2016}.

\paragraph{CNN Matching}
General purpose matching networks are also related to our work. \cite{Zagoruyko-15} used a matching CNN for patch matching, \cite{Dosovitskiy2015} used it for optical flow and \cite{Mayer_2016_CVPR} used it for stereo, optical flow and scene flow. Variants of networks~\cite{Dosovitskiy2015,Mayer_2016_CVPR} have been proposed that include a correlation layer explicitly; however, it is then used as a stack of features and followed by up-convolutions regressing the dense matching. Overall, these networks have a significantly larger number of parameters and require a lot of additional synthetic training data. 
%

\paragraph{Joint Training (CNN+CRF training)}
End-to-end training of CNNs and CRFs is helpful in many applications. The fully connected CRF~\cite{Kraehenbuehl2012}, performing well in semantic segmentation, was trained jointly in \cite{Chen2014a,Zheng2015} by unrolling iterations of the inference method (mean field) and backpropagating through them. Unfortunately, this model does not seem to be suitable for stereo because typical solutions contain slanted surfaces and not piece-wise constant ones (the filtering in~\cite{Kraehenbuehl2012} propagates information in fronto-parallel planes). Instead simple heuristics based on dynamic programming such as~SGM \cite{Hirschmueller2005} / MGM~\cite{Facciolo-15} are typically used in engineered stereo methods as post-processing. However they suffer from various artifacts as shown in~\cite{Facciolo-15}. A trained inference model, even a relatively simple one, such as dynamic programming on a tree~\cite{Psota_2015_ICCV}, can become very competitive. \citet{Scharstein07learnin} and \citet{Christopher-12-learnig} have considered training CRF models for stereo, linear in parameters. To the best of our knowledge, training of inference techniques with CNNs has not yet been demonstrated for stereo. We believe the reason for that is the relatively slow inference for models over pixels with hundreds of labels. Employing the method proposed in~\cite{Discrete-Continuous-16}, which is a variant of a LP-relaxation on the GPU, allows us to overcome this limitation. In order to train this method we need to look at a suitable learning formulation. Specifically, methods approximating marginals are typically trained with variants of approximate maximum likelihood~\cite{Alahari10a,KirillovSFZ0TR15,LinSRH15,Nowozin-13,Christopher-12-learnig,Scharstein07learnin}.
Inference techniques whose iteration can be differentiated can be unrolled and trained directly by gradient descent~\cite{Liu_2015_ICCV,Ochs2015,Ochs-16,Ranftl-14,Schwing-15,Tompson-14,Zheng2015}. Inference methods based on LP relaxation can be trained discriminatively, using a structured SVM approach~\cite{ChenSchwingICML2015,Franc-Laskov-11,Komodakis-11-train,Tsochantaridis-2005}, where parameters of the model are optimized jointly with dual variables of the relaxation (blended learning and inference). We discuss the difficulty of applying this technique in our setting (memory and time) and show that instead performing stochastic approximate subgradient descent is more feasible and practically efficient.

%% file: tex/model.tex
\section{CNN-CRF Model}
%

In this section we describe the individual blocks of our model (\cref{fig:model}) and how they connect.
\par
We consider the standard rectified stereo setup, in which epipolar lines correspond to image rows. Given the left and right images $I^0$ and $I^1$, the left image is considered as the {\em reference} image and for each pixel we seek to find a matching pixel of $I^1$ at a range of possible disparities. The disparity of a pixel $i \in \Omega = \dom I^0$ is represented by a discrete label $x_i \in \L = \{0,\dots L-1\}$.
\par
The \unarycnn extracts dense image features for $I^0$ and $I^1$ respectively, denoted as $\phi^0 = \phi(I^0; \theta_1)$ and $\phi^1 = \phi(I^1; \theta_1)$. 
Both instances of the \unarycnn in \cref{fig:model} share the parameters $\theta_1$. For each pixel, these extracted features are then correlated at all possible disparities to form a correlation-volume (a matching confidence volume) $p \colon \Omega \times \L \to [0,1]$.
The confidence 
$p_i(x_i)$ is interpreted as how well a window around pixel $i$ in the first image $I^0$ matches to the window around pixel $i+x_i$ in the second image $I^1$. Additionally, the reference image $I^0$ is used to estimate contrast-sensitive edge weights either using a predefined model based on gradients, or using a trainable pairwise CNN. The correlation volume together with the pairwise weights are then fused by the CRF inference, optimizing the total cost.
\subsection{Unary CNN}
We use 3 or 7 layers in the \unarycnn and 100 filters in each layer. The filter size of the first layer is $(3 \times 3)$ and the filter size of all other layers is $(2 \times 2)$. We use the $\tanh$ activation function after all convolutional layers. Using $\tanh$ i) makes training easier, \ie, there is no need for intermediate (batch-)normalization layers and ii) keeps the output of the correlation-layer bounded. Related works~\cite{BailerVS16,BHW10} have also found that $\tanh$ performs better than ReLU for patch matching with correlation.
\subsection{Correlation}
The cross-correlation of features $\phi^0$ and $\phi^1$ extracted from the left and right image, respectively, is computed as
\begin{equation}
p_{i}(k) = \frac{e^{\< \phi^0_{i}, \phi^1_{i+k} \>}}{\sum_{j\in\L} e^{\< \phi^0_{i}, \phi^1_{i+j} \>}}  \quad \forall i \in \Omega, \forall k \in \L.
\end{equation}
Hence, the correlation layer outputs the softmax normalized scalar products of corresponding feature vectors. In practice, the normalization fixes the scale of our unary-costs which helps to train the joint network. Since the correlation function is homogeneous for all disparities, a model trained with some fixed number of disparities can be applied at test time with a different number of disparities.
The {\em pixel-wise independent estimate} of the best matching disparity
\begin{equation}\label{argmax}
x_i \in \arg\max_k ~ p_i(k)
\end{equation}
is used for the purpose of comparison with the full model.

\input{tex/crf.tex}

\subsection{Pairwise CNN}
In order to estimate edge weights with a pairwise CNN, we use a 3-layer network. We use 64 filters with size ($3\times3$) and the $\tanh$ activation function in the first two layers to extract some suitable features. The third layer maps the features of pixel $i$ to weights $(w_{ij} \mid ij\in \E)$ corresponding to the two edge orientations, where we use the absolute value function as activation. This ensures that the pairwise costs are always larger than $0$ and that our \pairwisecnn has the ability to scale the output freely. In practice this is desirable because it allows us to automatically learn the optimal trade-off between data-fidelity and regularization. The parameters of this network will be denoted as $\theta_2$. The weights $w$ can be stored as a $2$-channel image (one channel per orientation). They generalize over the manually defined contrast-sensitive weights defined in \eqref{eq:fixed-pairwise-model} in the pairwise-terms $f_{ij}$ \eqref{eq:pairwise-model}. Intuitively, this means the pairwise network can learn to apply the weights $w$ adaptively based on the image content in a wider neighborhood.
The values $P_1, P_2$ remain as global parameters.
\cref{fig:pairwise-example} shows an example output of the \pairwisecnn.
\begin{figure}[t]
\centering
\includegraphics[width=0.5\linewidth]{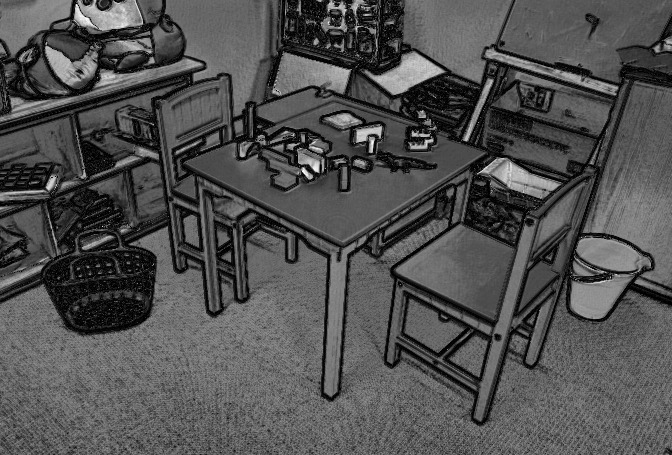}
\includegraphics[width=0.5\linewidth]{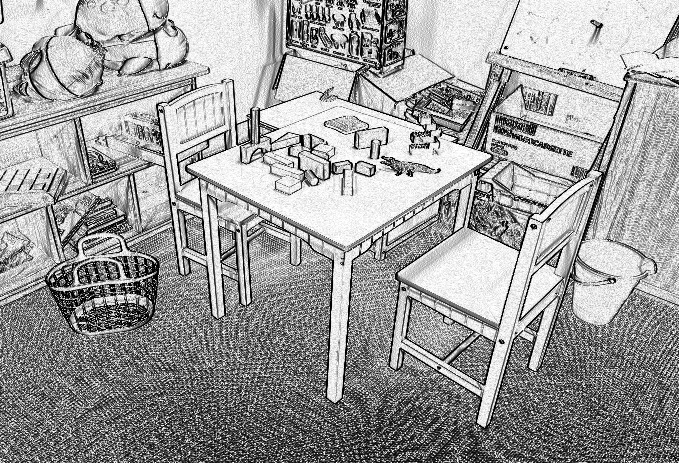}

\caption{Learned vs fixed pairwise costs: Visualization of the pairwise costs between two neighboring pixels in horizontal direction using the learned \pairwisecnn (left) and a fixed edge-function (right). Dark pixels indicate a low cost for changing the label and bright pixels indicate a high cost for a label-switch. Note, how the dark pixels follow object outlines (where depth discontinuities are likely) and how texture-edges tend to be suppressed (\eg, on the floor) in the learned version.}
\label{fig:pairwise-example}
\end{figure}

%% file: tex/crf.tex

\subsection{CRF}
\label{sec:CRF}
The CRF model optimizes the total cost of complete disparity labelings, 
\begin{equation}\label{eq:crf}
\min_{x \in \mathcal{X}}\big( f(x) := \sum_{i \in \mathcal{V}} f_i(x_i) + \sum_{ij\in\mathcal{E}} f_{ij}(x_i,x_j) \big).
\end{equation}
where $\mathcal{V}$ is the set of all nodes in the graph, \ie, the pixels, $\mathcal{E}$ is the set of all edges and $\X = \L^\V$ is the space of labelings.
Unary terms $f_i \colon \L \to \Real$ are set as $f_i(k) = -p_i(k)$, the matching costs. The pairwise terms $f_{ij} \colon \L\times\L \to \Real$ implement the following model:
\begin{equation}\label{eq:pairwise-model}
f_{ij}(x_i,x_j) = w_{ij} \rho(|x_i - x_j|;P_1,P_2).
\end{equation}
The weights $w_{ij}$ may be set either as manually defined contrast-sensitive weights~\cite{Boykov00}:
\begin{align}\label{eq:fixed-pairwise-model}
w_{ij} = \exp(-\alpha |I_i-I_j|^\beta) \tab\tab \forall ij\in\E,
\end{align}
allowing cheaper disparity jumps across strong image gradients, or using the learned model of the \pairwisecnn. The function $\rho$ is a robust penalty function defined as
\begin{equation}\label{eq:interaction model}
\rho(|x_i - x_j|) =
\begin{cases}
0 & \text{if } |x_i - x_j| = 0, \\
P_1 & \text{if } |x_i - x_j| = 1, \\
P_2 & \text{otherwise},
\end{cases}
\end{equation}
popular in stereo~\cite{hirschmuller2011semi}. Cost $P_1$ penalizes small disparity deviation of one pixel representing smooth surfaces and $P_2$ penalizes larger jumps representing depth discontinuities.
We use only pairwise-interactions on a 4-connected grid.
\paragraph{Inference} Although the direct solution of~\eqref{eq:crf} is intractable~\cite{Li2016}, there are a number of methods to perform approximate inference~\cite{ChenSchwingICML2015,Kolmogorov-06-convergent-pami} as well as related heuristics designed specifically for stereo such as~\cite{Facciolo-15,hirschmuller2011semi}. We apply our dual minorize-maximize method (\protect\anchor[{\texttt{Dual\_MM}}]{DMM})~\cite{Discrete-Continuous-16}, which is sound because it is based on LP-relaxation, similar to TRW-S~\cite{Kolmogorov-06-convergent-pami}, and massively parallel, allowing a fast GPU implementation.
\par We give a brief description of \DMM, which will also be needed when considering training.
Let $f$ denote the concatenated {\em cost vector} of all unary and pairwise terms $f_i, f_{ij}$. The method starts from a decomposition of $f$ into horizontal and vertical chains, $f = f^{1} + f^{2}$ (namely, $f^{1}$ includes all horizontal edges and all unary terms and $f^{2}$ all vertical edges and zero unary terms). 
The value of the minimum in~\eqref{eq:crf} is lower bounded by
\begin{align}\label{eq:crf-dual}
\max_\lambda \big( D(\lambda) := \min_{x^1}(f^{1} + \lambda)(x^1) + \min_{x^2}(f^{2} - \lambda)(x^2) \big),
\end{align}
where $\lambda$ is the vector of Lagrange multipliers corresponding to the constraint $x^1 = x^2$.
The bound $D(\lambda)\leq \eqref{eq:crf}$ holds for any $\lambda$, however it is tightest for the optimal $\lambda$ maximizing the sum in the brackets.
The \DMM algorithm performs iterations towards this optimum by alternatively updating $\lambda$ considering at a time either all vertical or horizontal chains, processed in parallel. Each update monotonously increases the lower bound~\eqref{eq:crf-dual}. The final solution is obtained as
\begin{align}\label{eq:argmin-reparametrized}
x_i \in \argmin_{k} (f^1_i + \lambda_i)(k),
\end{align}
\ie, similar to~\eqref{argmax}, but for the reparametrized costs $f^{1} + \lambda$. If the inference has converged and the minimizer $x_i$ in~\eqref{eq:argmin-reparametrized} is unique for all $i$, then $x$ is the optimal solution to the energy minimization~\eqref{eq:crf}~\cite{Komodakis-subgradient,Werner-PAMI07}. 


%% file: tex/training.tex

\section{Training}
\label{sec:training}
One major goal of this work is the end-to-end training of the complete model in \cref{fig:model}. 
For the purpose of comparison of different components we train 3 types of models, of increasing generality:
\begin{itemize}
\item Pixel-wise \unarycnn: model in which CRF interactions are set to zero and \pairwisecnn is switched off.
\item Joint \unarycnn+CRF model in which the \pairwisecnn is fixed to replicate exactly the contrast-sensitive model~\eqref{eq:fixed-pairwise-model}. Trained parameters are: \unarycnn and global parameters $P_1,P_2$.
\item Joint model with trained \unarycnn and \pairwisecnn (=complete model). Trained Parameters are:  \unarycnn , \pairwisecnn and global parameters $P_1, P_2$.
\end{itemize}

\subsection{Training Unary CNN in the Pixel-wise Model}
\label{subsec:training_unary}
For the purpose of comparison, we train our \unarycnn in a pixel-wise mode, similarly to~\cite{Chen2015,Luo2016,Zbontar2016}. For this purpose we set the CRF interactions to zero (\eg, by letting $P_1=P_2=0$), in which case the resulting decision degenerates to the pixel-wise independent $\argmax$ decision rule~\cref{argmax}. 
Training such models can be formulated in different ways, using gradient of the likelihood / cross-entropy~\cite{Luo2016,Zbontar2015a}, reweighed regression~\cite{Chen2015} or hinge loss~\cite{Zbontar2015}. Following~\cite{Luo2016,Zbontar2015a} we train parameters of the \unarycnn $\theta_1$ using the cross-entropy loss, 
\begin{equation}
\min_{\theta_1} \sum_{i \in \Omega} \sum_{k \in \mathcal{X}} p^{gt}_i(k) \log p_i(k; \theta_1),
\label{eq:unaryLoss}
\end{equation}
where $p^{gt}_i(k)$ is the one-hot encoding of the ground-truth disparity for the $i$-th pixel.
%
%
\subsection{Training Joint Model}
\label{subsec:training_joint}
We apply the structured support vector machine formulation, also known as the maximum margin Markov network~\cite{Taskar03max-marginmarkov,Tsochantaridis-2005}, in a non-linear setting. After giving a short overview of the SSVM approach we discuss the problem of learning when no exact inference is possible. We argue that the blended learning and inference approach of~\cite{ChenSchwingICML2015,Komodakis-11-train} is not feasible for models of our size. We then discuss the proposed training scheme approximating a subgradient of a fixed number of iterations of~\DMM.
\par
\paragraph{SSVM} Assume that we have a training sample consisting of an input image pair $I=(I^0,I^1)$ and the true disparity $x^*$. 
Let $x$ be a disparity prediction that we make. We consider an additive loss function
\begin{align}\label{eq:loss}
l(x,x^*) = \sum_{i} l_i(x_i,x^*_i),
\end{align}
where the pixel loss $l_i$ is taken to be $l_i(x_i,x^*_i) = \min(|x_i - x^*_i|,\tau)$, appropriate in stereo reconstruction. The empirical risk is the sum of losses~\eqref{eq:loss} over a sample of several image pairs, however for our purpose it is sufficient to consider only a single image pair.
When the inference is performed by the CRF \ie, the disparity estimate $x$ is the minimizer of~\eqref{eq:crf}, training the optimal parameters $\theta = (\theta_1, \theta_2, P_1, P_2)$ can be formulated in the form of a {\em bilevel optimization}:
\begin{subequations}
\begin{eqnarray}\label{learning-bilevel}
&& \min_\theta l(x,x^*) \\
\label{learning-bilevel-constraint}
&& \text{s.t. } x \in \arg \min_{x \in \mathcal{X}} f(x; \theta).
\end{eqnarray}
\end{subequations}


Observe that any $x\in\argmin f(x)$ in~\eqref{learning-bilevel-constraint} necessarily satisfies $f(x) \leq f(x^*)$. Therefore, for any $\gamma>0$, the scaled loss $\gamma l(x,x^*)$ can be upper-bounded by
\begin{subequations}
\begin{eqnarray}\label{hinge-loss-ssvm}
&& \max_{x: ~f(x) \leq f(x^*)} \gamma l(x, x^*) \\
&& \leq
\max_{x: ~f(x) \leq f(x^*)} \left [f(x^*) - f(x) + \gamma l(x, x^*) \right] \\
\label{hinge-loss-ssvm-c}
&&\leq \max_x \left[ f(x^*) - f(x) + \gamma l(x, x^*)\right].
\end{eqnarray}
\end{subequations}
A subgradient of~\eqref{hinge-loss-ssvm-c} \wrt $(f_i \mid i\in\V)$ can be chosen as
\begin{align}\label{eq:crf-grad-exact}
\delta(x^*) - \delta(\bar x),
\end{align}
where $\delta(x)_i$ is a vector in $\Real^{\L}$ with components $(\leftbb x_i = k\rightbb \mid k\in\L)$, \ie the 1-hot encoding of $x_i$, and 
$\bar x$ is a (generally non-unique) solution
to the {\em loss augmented inference} problem
\begin{align}\label{eq:loss-augmented}
\bar x \in \argmin_x \big[ \bar f (x) := f(x) - \gamma l(x,x^*) \big].
\end{align}
In the case of an additive loss function, problem~\eqref{eq:loss-augmented} is of the same type as~\eqref{eq:crf} with adjusted unary terms.

We facilitate the intuition of why the SSVM chooses the most violated constraint by rewriting the hinge loss~\eqref{hinge-loss-ssvm-c} in the form
\begin{align}\label{ssvm-hinge-xi}
& \min\{\xi \in \Real \mid (\forall x)\ \ \xi \geq f(x^*)  - f(x) + \gamma l(x,x^*)\},
\end{align}
which reveals the large margin separation property: the constraint in~\eqref{ssvm-hinge-xi} tries to ensure that the training solution $x^*$ is better than all other solutions by a margin $\gamma l(x,x^*)$ and the most violated  constraint sets the value of slack $\xi$. The parameter $\gamma$ thus controls the margin: a large margin may be beneficial for better generalization with limited data. Finding the most violated constraint in~\eqref{ssvm-hinge-xi} is exactly the loss-augmented problem~\eqref{eq:loss-augmented}.



\paragraph{SSVM with Relaxed Inference} An obstacle in the above approach is that we cannot solve the loss-augmented inference~\eqref{eq:loss-augmented} exactly.
However, having a method solving its convex relaxation, we can integrate it as follows. Applying the decomposition approach to~\eqref{eq:loss-augmented} yields a lower bound on the minimization: $\eqref{eq:loss-augmented} \geq$
\begin{align}\label{LP-augmented-loss}
\bar D(\lambda) := \min_{x^1}(\bar f^{1} + \lambda)(x^1) + \min_{x^2}(\bar f^{2} - \lambda)(x^2)
\end{align}
for all $\lambda$. Lower bounding~\eqref{eq:loss-augmented} like this results in an upper-bound of the loss $\gamma l(x,x^*)$ and the hinge loss~\eqref{hinge-loss-ssvm}:
\begin{align}\label{LP-SSVM-bound}
\gamma l(x,x^*) \leq \eqref{hinge-loss-ssvm} \leq f(x^*) -\bar D(\lambda).
\end{align}
The bound is valid for any $\lambda$ and is tightened by maximizing $D(\lambda)$ in $\lambda$. The learning problem on the other hand minimizes the loss in $\theta$. Tightening the bound in $\lambda$ and minimizing the loss in $\theta$ can be written as a joint problem
\begin{align}\label{LP-SSVM-bound-min}
\min_{\theta, \lambda} f(x^*; \theta) - \bar D(\lambda; \theta).
\end{align}
Using this formulation we do not need to find an optimal $\lambda$ at once; it is sufficient to make a step towards minimizing it. This approach is known as blended learning and inference~\cite{ChenSchwingICML2015,Komodakis-11-train}. It is disadvantageous for our purpose for two reasons: i) at the test time we are going to use a fixed number of iterations instead of optimal $\lambda$ ii) joint optimization in $\theta$ and $\lambda$ in this fashion will be slower and iii) it is not feasible to store intermediate $\lambda$ for each image in the training set as $\lambda$ has the size of a unary cost volume.
\paragraph{Approximate Subgradient}
We are interested in a subgradient of~\eqref{LP-SSVM-bound} after a fixed number of iterations of the inference method, \ie, training the unrolled inference. A suboptimal $\lambda$ (after a fixed number of iterations) will generally vary when the CNN parameters $\theta$ and thus the CRF costs $f$ are varied. While we do not fully backtrack a subgradient of $\lambda$ (which would involve backtracking dynamic programming and recursive subdivision in \DMM) we can still inspect its structure and relate the subgradient of the approximate inference to that of the exact inference. 
\par
\begin{restatable}{proposition}{PUnaryGrad}\label{P:unary-grad}
Let $\bar x^1$ and $\bar x^2$ be minimizers of horizontal and vertical chain subproblems in~\eqref{LP-augmented-loss} for a given $\lambda$. Let $\Omega_{{\neq}}$ be a subset of nodes for which $\bar x^1_i \neq \bar x_i^2$. Then a subgradient $g$ of the loss upper bound~\eqref{LP-SSVM-bound} \wrt $f_\V = (f_i \mid i\in \V)$ has the following expression in components
\begin{align}\label{subgrad-unary}
g_i(k) &= \big(\delta(x^*) - \delta(\bar x^1)\big)_i(k)\\
\notag & + \sum_{j \in \Omega_{\neq}}\big(J_{ij}(k,\bar x_i^2) - J_{ij}(k,\bar x_i^1) \big),
\end{align}
\end{restatable}
\noindent where $J_{ij}(k,l)$ is a sub-Jacobian (matching $\frac{d \lambda_j(l)}{d f_i(k)}$ for a subset of directions $d f_i(k)$).
See \cref{sec:suppl_detail} for more details.
\par
We conjecture that when the set $\Omega_{\neq}$ is small, for many nodes the contribution of the sum in~\eqref{subgrad-unary} will be also small, while the first part in~\eqref{subgrad-unary} matches the subgradient with exact inference~\eqref{eq:crf-grad-exact}.
\par
\begin{proposition}For training the abbreviate inference with dual decomposition such as \DMM, we calculate the minimizer $\bar x^1$ after a fixed number of iterations and approximate the subgradient as
$\delta(x^*) - \delta(\bar x^1)$.
\end{proposition}
%
\par
The assumption for the learning to succeed is to eventually have most of the pixels in agreement. 
The inference method works towards this by adjusting $\lambda$ such that the constraints $x^1_i=x^2_i$ are satisfied. We may expect in practice that if the data is not too ambiguous this constraint will be met for a large number of pixels already after a fixed number of iterations. 
A good initialization of unary costs, such as those learned using the pixel-wise only method can help to improve the initial agreement and to stabilize the method.
%

\subsection{Training Unary and Pairwise CNNs in Joint Model}
\label{sec:jointTraining}
To make the pairwise interactions trainable, we need to compute a  subgradient \wrt $\w_{ij}$, $P_1$, $P_2$. 
We will compute it similarly to the unary terms assuming exact inference, and then just replace the exact minimizer $\bar x$ with an approximate $\bar x^1$. A subgradient of~\eqref{hinge-loss-ssvm-c} is obtained by choosing a minimizer $\bar x$ and evaluating the gradient of the minimized expression. Components of the later are given by
\begin{subequations}
\begin{align}
& \textstyle \frac{\partial}{\partial w_{ij}} = \rho(|x_i^*{-}x_j^*|;P_{1,2}) - \rho(|\bar x_i - \bar x_j|;P_{1,2}),\\
& \textstyle \frac{\partial}{\partial P_{1}} = \sum_{ij} w_{ij}( \leftbb|x_i^*{-}x_j^*| = 1\rightbb - \leftbb|\bar{x}_i{-}\bar{x}_j| = 1\rightbb),\\
& \textstyle \frac{\partial}{\partial P_{2}} = \sum_{ij} w_{ij}(\leftbb|x_i^*{-}x_j^*| > 1\rightbb - \leftbb|\bar{x}_i{-}\bar{x}_j| > 1\rightbb).
\end{align}
\end{subequations}
We thus obtain an end-to-end trainable model without any hand-crafted parameters, except for the hyper-parameters controlling the training itself.
%
%

%% file: tex/implementation_details.tex

\subsection{Implementation Details}
We trained our models using Theano \cite{Bergstra2010} with stochastic gradient descent and momentum. For training the model without pairwise costs we set the learn rate to 1\e{-2}, for all other models we set the learn rate to 1\e{-6}. Before feeding a sample into our model we normalize it such that it has zero-mean and unit-variance. We additionally correct the rectification for Middlebury samples. Our full model is trained gradually. We start by training the models with lower complexity and continue by training more complex models, where we reuse previously trained parameters and initialize new parameters randomly. Since we use full RGB images for training, we have to take care of occlusions as well as invalid pixels, which we mask out during training.
Additionally, we implemented the forward pass using C++/CUDA in order to make use of our trained models in a real-time environment in a streaming setting.
We achieve 3-4 frames per second with our fully trained 3-layer model using an input-size of $640 \times 480$ pixels\footnote{A detailed breakdown of the timings can be found in the supplementary material.}.

%% file: tex/experiments.tex

\section{Experiments}
In this section we test different variants of our proposed method. In order not to confuse the reader, we use the following naming convention: {\em CNNx} is the $\argmax$ output of a network trained as described in \cref{subsec:training_unary}; {\em CNNx+CRF} is the same network with \DMM as post-processing; {\em CNNx+CRF+Joint} is the jointly trained network described in \cref{subsec:training_joint} and {\em CNNx+CRF+Joint+PW} is the fully trained method described in \cref{sec:jointTraining}. $x$ represents the number of layers in the CNN.

\subsection{Benchmark Data Sets}\label{subsec:benchmark_datasets}
We use two stereo benchmark datasets for our experiments: Kitti 2015~\cite{Menze2015} and Middlebury V3~\cite{Scharstein2014}. Both benchmarks hold out the {\bf test} set, where the ground truth is not accessible to authors. We call examples with ground truth available that can be used for training/validation the {\bf design} set and split it randomly into 80\% {\bf training} set and 20\% {\bf validation} set. This way we obtain $160+40$ examples for Kitti and $122+31$ examples for Middlebury (including additionally provided images with different lightings, exposures and perfectly/imperfectly rectified stereo-pairs). The used error metric in all experiments is the percent of pixels with a disparity difference above $x$ pixels ({\em badx}). 


\subsection{Performance of Individual Components}\label{sec:components-influence}
In this experiment we measure the performance improvement when going from {\em CNNx} to the full jointly trained model.
Since ground-truth of the test data is not available to us, this comparison is conducted on the complete design set.
The results are shown in~\cref{tab:evalAll}.
 This experiment demonstrates that an optimization or post-processing is necessary, since the direct output of all tested CNNs (after a simple point-wise minimum search in the cost volume) contains too many outliers to be used directly. A qualitative comparison on one of the training images of Middlebury is depicted in \cref{fig:middleburryQualitative}. One can observe that the quality of the CNN-only method largely depends on the number of layers, whereas the CNN+CRF versions achieve good results even for a shallow CNN.
\cref{tbl:online} additionally shows the error metrics {\em bad\{2,3,4\}} on the design set of Kitti, because these error metrics cannot be found online. 


\subsection{Benefits of Joint Training}\label{sec:comparison-cnns}


In this experiment, we compare our method to two recently proposed stereo matching methods based on CNNs, the {\em MC-CNN} by Zbontar and LeCun~\cite{Zbontar2016} and the {\em Content-CNN} by \citet{Luo2016}. To allow a fair comparison of the methods, we disable all engineered post-processing steps of \cite{Luo2016,Zbontar2016}. We then unify the post-processing step by adding our CRF on top of the CNN outputs. We evaluate on the whole design set since we do not know the train/test split of the different methods.
In favor of the compared methods, we individually tune the parameters $P_1,P_2,\alpha,\beta$ of the CRF for each method using grid search.
The results are shown in \cref{tab:evalAll}. While the raw output of our CNN is inferior to the compared methods, the post-processing with a CRF significantly decreases the difference in performance.
Joint training of our CNN+CRF model further improves the performance, despite using a relatively shallow network with fewer parameters. Specifically, our full joint model with 7 layers has 281k parameters, while the networks~\cite{Luo2016,Zbontar2016} have about 700k and 830k parameters, respectively.

\begin{figure}[t]
\centering
\vspace{-5pt}
\begin{tabular}{c}
\small Input \\
\includegraphics[width=0.425\linewidth]{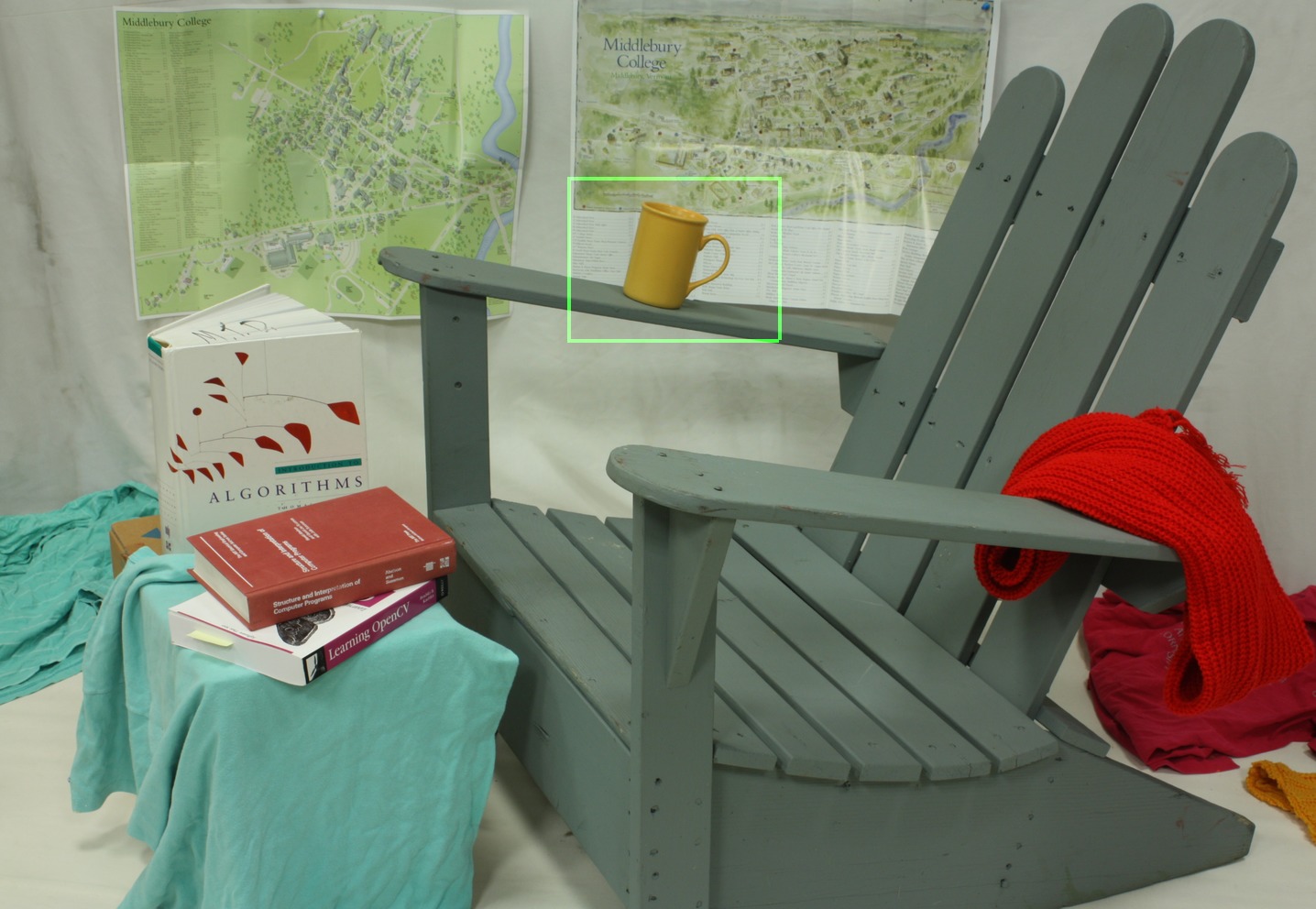}
\end{tabular}%
\begin{tabular}{p{0.575\linewidth}}
\begin{tabular}{C{0.33\linewidth}C{0.33\linewidth}C{0.33\linewidth}}
\small CNN & \small +CRF & \small +Joint+PW
\end{tabular}\\
\includegraphics[width=\linewidth]{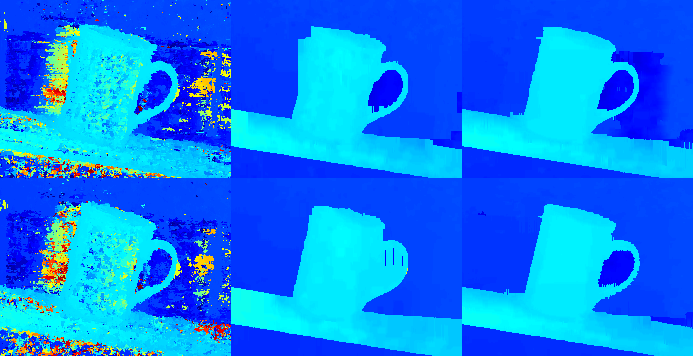}
\end{tabular}\\
\caption{Qualitative comparison of \unarycnn, CNN+CRF and CNN+CRF+Joint+PW on the Middlebury benchmark. Zoom-in of disparity with 3 layers (top) and 7 layers (bottom). Note how the jointly trained models inpaint occlusions correctly.
}
\label{fig:middleburryQualitative}
\end{figure}

\begin{table}[t]
\centering
\setlength{\tabcolsep}{3pt}
\small
\begin{tabular}{llcccc}
\toprule
\textbf{Benchmark}& \textbf{Method} & \textbf{CNN} & \textbf{+CRF} & \textbf{+Joint} & \textbf{+PW} \\
\midrule
\multirow{2}{*}{Middlebury}& CNN3 & 23.89 & 11.18 & 9.48 & 9.45 \\
                           & CNN7 & 18.58 & 9.35  & 8.05 & 7.88 \\ \midrule
\multirow{4}{*}{Kitti 2015}& CNN3 & 28.38 & 6.33  & 6.11 & 4.75 \\
                           & CNN7 & 13.08 & 4.79  & 4.60 & 4.04 \\
              &    \cite{Luo2016} & 5.99  & 4.31  & -    & -    \\
              &\cite{Zbontar2016} & 13.56 & 4.45  & -    & -    \\ \bottomrule
\end{tabular}
\caption{Influence of the individual components of our method (\cref{sec:components-influence}) and comparison with~\cite{Luo2016,Zbontar2016} without post-processing (\cref{sec:comparison-cnns}). 
Standard error metrics ({\em bad4} on official training data for Middlebury and {\em bad3} on the {\bf design} set for Kitti) are reported.
}
\label{tab:evalAll}
\end{table}


\begin{figure*}[t]
\centering
\vspace{-5pt}
\resizebox{0.75\linewidth}{!}{
\begin{minipage}{0.68\textwidth}
\centering
{\normalsize Middlebury}\\
\setlength{\tabcolsep}{1.5pt}
\footnotesize
\begin{tabu}{lc|c|*{15}{c}|c}
\toprule
{\bf Method} & \sw{\bf \begin{tabular}{l} Average\\ performance\end{tabular}} &  \sw{\bf \begin{tabular}{l} Time [sec] \end{tabular}} & \sw{Australia} & \sw{AustraliaP} & \sw{Bicycle2} & \sw{Classroom2} & \sw{Classroom2E} & \sw{Computer} & \sw{Crusade} & \sw{CrusadeP} & \sw{Djembe} & \sw{DjembeL} & \sw{Hoops} & \sw{Livingroom} & \sw{Newkuba} & \sw{Plants} & \sw{Staircase} & \sw{\bf Metric}\\
\midrule
\cite{Zbontar2016} fst & 22.4	&	\textbf{1.69} & 22.0	&	20.3	&	12.7	&	28.8	&	42.6	&	9.82	&	28.7	&	25.1	&	5.07	&	32.0	&	23.3	&	16.5	&	30.6	&	25.5	&	34.1 &	\\
\cite{Zbontar2016} acc. & 21.3	&	150 & 20.8	&	19.6	&	9.6	&	28.6 &	67.4	&	7.67	&	23.2	&	15.7	&	8.49	&	31.8	&	\textbf{16.7}	&	13.9	&	38.8	&	18.7	&	28.6 & \multirow{4}{*}{\sw{RMS}}	\\
 \cite{Barron2016} & 15.0	&	188 & 18.4	&	18.1	&	\textbf{8.72}	&	\textbf{9.06}  &	19.9	&	\textbf{6.52}	&	24.2	&	25.7	&	\textbf{3.91}	&	\textbf{12.7}	&	24.7	&	\textbf{9.58}	&	\textbf{17.9}	&	\textbf{17.5} & 	\textbf{17.9} &	\\
Ours & \textbf{14.4} & 4.46 & \textbf{15.9} & \textbf{16.2} & 10.7 & 10.3 & \textbf{11.2} & 14.0 & \textbf{13.7} & \textbf{13.1} & 4.11 & 14.3 & 19.2 & 11.9 & 22.5 & 20.6 & 25.5 & \\
\midrule
\cite{Zbontar2016} fst & 9.47	&	\textbf{1.69} & 7.35	&	5.07	&	7.18	&	4.71	&	16.8	&	8.47	&	\textbf{7.37}	&	\textbf{6.97}	&	2.82	&	20.7	&	17.4	&	15.4	&	15.1	&	7.9	&	12.6 &	\\
\cite{Zbontar2016} acc. & \textbf{8.29}	&	150 & 5.59	&	4.55	&	\textbf{5.96}	&	\textbf{2.83}	&	11.4	&	\textbf{8.44}	&	8.32	&	8.89	&	\textbf{2.71}	&	16.3	&	\textbf{14.1}	&	13.2	&	\textbf{13.0}	&	\textbf{6.40}	&	11.1 & \multirow{4}{*}{\sw{bad2}}	\\
 \cite{Barron2016} & 8.62	&	188 & 6.05	&	5.16	&	6.24	&	3.27  &	\textbf{11.1}	&	8.91	&	8.87	&	9.83	&	3.21	&	\textbf{15.1}	&	15.9	&	\textbf{12.8}	&	13.5	&	7.04 & 	\textbf{9.99} &	\\
Ours & 12.5 & 4.46 & \textbf{4.09} & \textbf{3.97} & 8.44 & 6.93 & \textbf{11.1} & 13.8 & 19.5 & 19.0 & 3.66 & 17.0 & 18.2 & 18.0 & 21.0 & 7.29 & 17.8 & \\
\bottomrule
\end{tabu}
\end{minipage}}
{\small
\begin{tabular}{c}
{\normalsize Kitti 2015}\\
\setlength{\tabcolsep}{1.5pt}
\begin{tabular}{lccc}
\toprule
{\bf Method} & {\bf Non-occ} & {\bf All} & {\bf Time}\\ \midrule
\cite{Mayer_2016_CVPR} & 4.32 & 4.34 & \textbf{0.06}s \\
\cite{Luo2016} & 4.00 & 4.54 & 1s\\
\cite{Zbontar2016} acc. & 3.33 & 3.89 & 67s\\
\cite{Seki2016} & \textbf{2.58} & \textbf{3.61} & 68s\\
Ours & 4.84 & 5.50 & 1.3s\\
\bottomrule
\toprule
{\bf Train err.} & $\textbf{bad2}$ & $\textbf{bad3}$ & $\textbf{bad4}$ \\
\midrule
\cite{Luo2016}\footnotemark[3]      &  7.39 & 4.31 & 3.14 \\
\cite{Zbontar2016}\footnotemark[3]   &  11.4 & 4.45 & \textbf{2.93} \\
Ours                     &  \textbf{6.01} & \textbf{4.04} & 3.15 \\
\bottomrule
\end{tabular}
\end{tabular}
}
\\
\captionof{table}{
Performance in benchmark {\bf test} sets as of time of submission. For both benchmarks, we compare our results against work that is based on CNNs for matching costs and accepted for publication.
We report the respective standard error metric {\em bad2} for the Middlebury- and {\em bad3} for the Kitti benchmark. The bottom table for Kitti shows a comparison of the training error with different error metrics {\em badx}. 
\label{tbl:online}
}
\end{figure*}

\begin{figure}[ht]
\centering
\includegraphics[width=0.49\columnwidth]{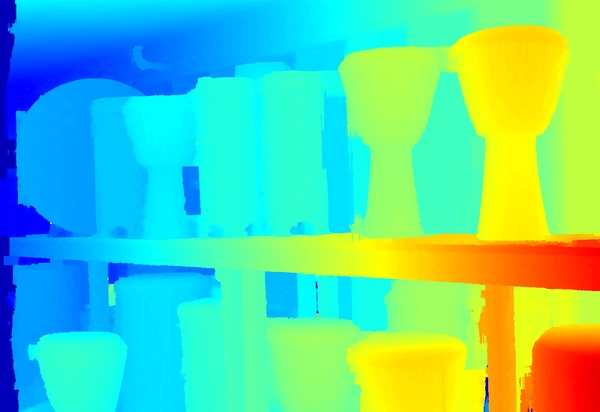}
\includegraphics[width=0.49\columnwidth]{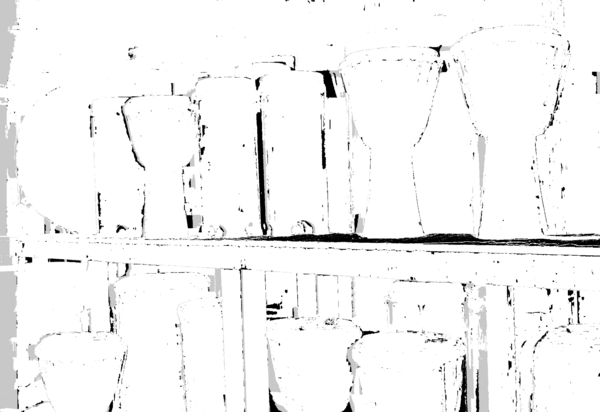}\\
\includegraphics[width=0.49\columnwidth]{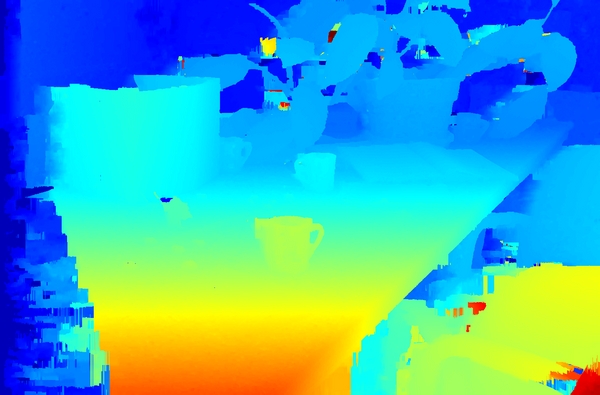}
\includegraphics[width=0.49\columnwidth]{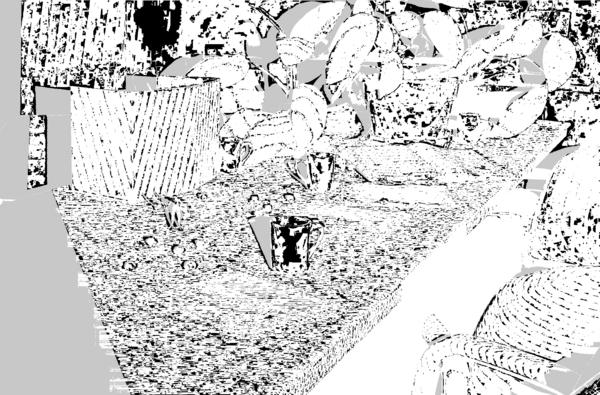}\\
\caption{Qualitative comparison on selected {\em test} images (from top to bottom: \textit{Djembe} and \textit{Crusade}) of the Middlebury Stereo Benchmark. The left column shows the generated disparity images in false color, the right column the bad2 error image, where white = error smaller than 2 disparities, grey = occlusion and black = error greater than 2 disparities.}
\label{fig:middleburyTestQualitative}
\end{figure}

\begin{figure}[t]
\centering
\includegraphics[width=0.66\columnwidth]{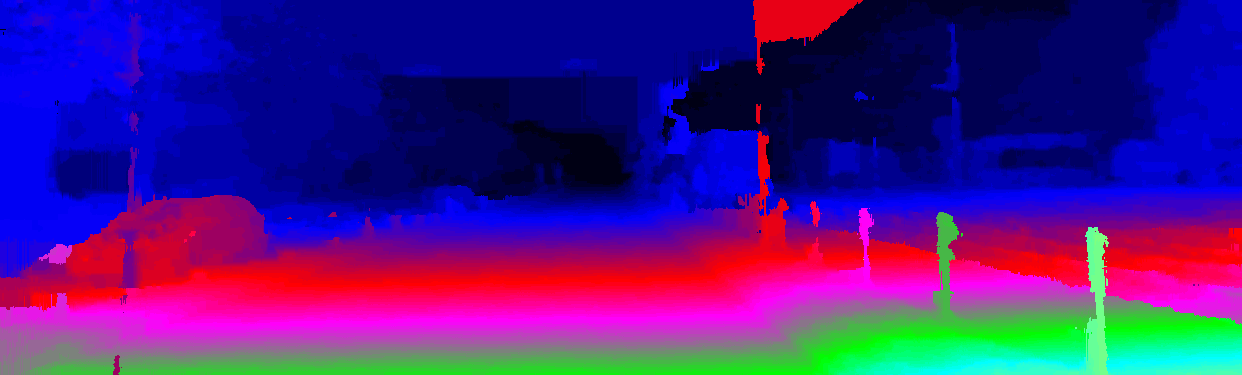}
\begin{minipage}[b][][c]{0.32\columnwidth}
\includegraphics[width=\columnwidth]{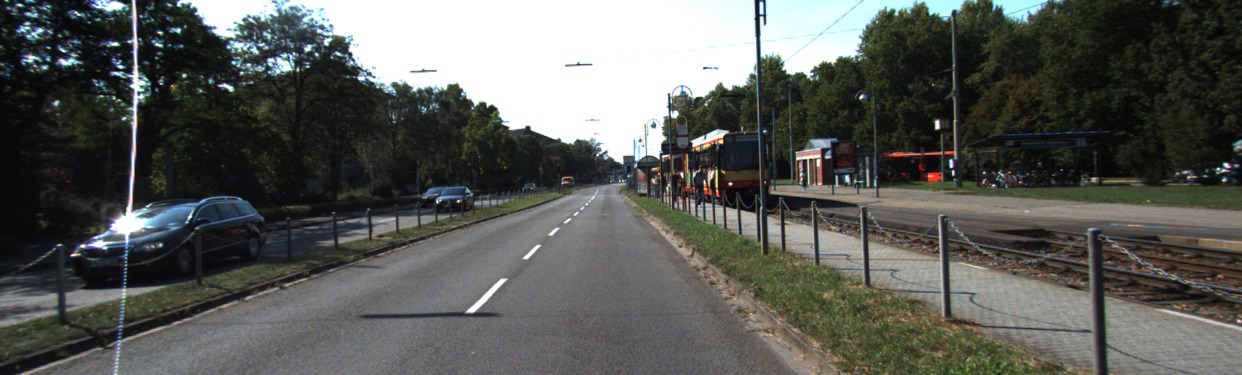}\\
\includegraphics[width=\columnwidth]{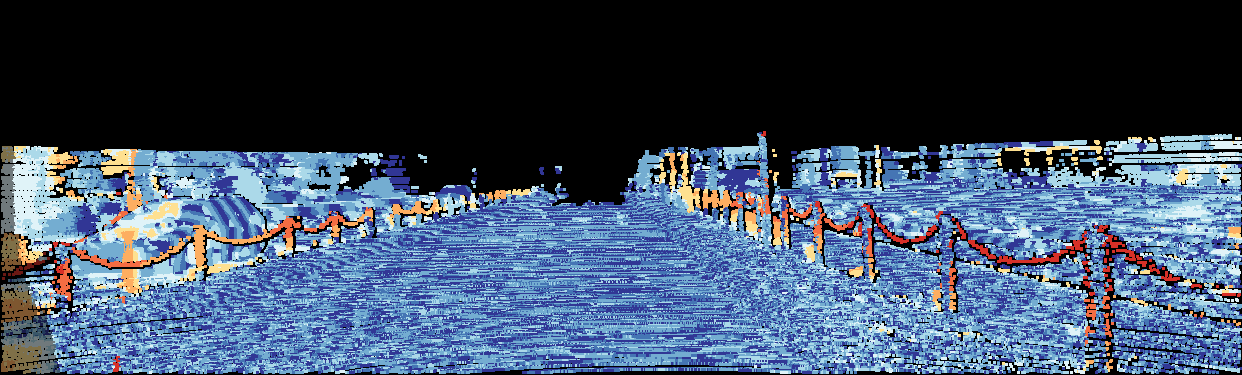}
\end{minipage}
\includegraphics[width=0.66\columnwidth]{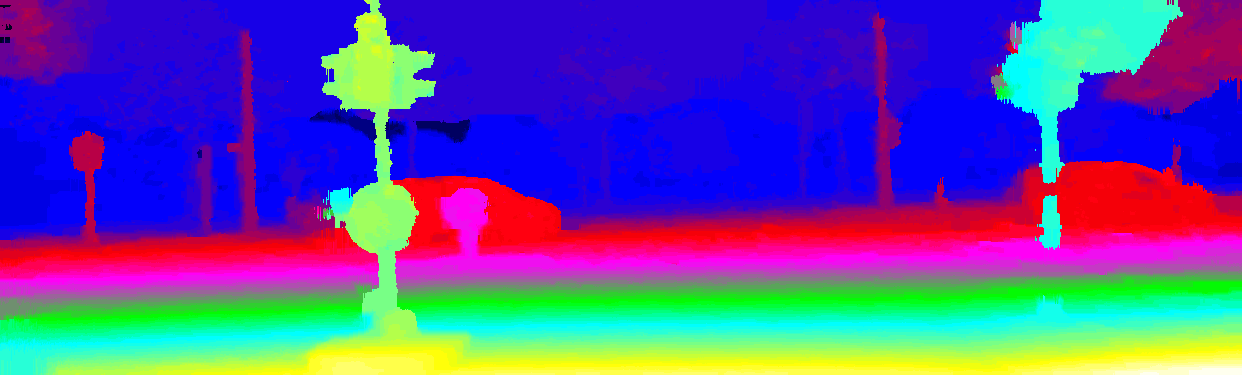}
\begin{minipage}[b][][c]{0.32\columnwidth}
\includegraphics[width=\columnwidth]{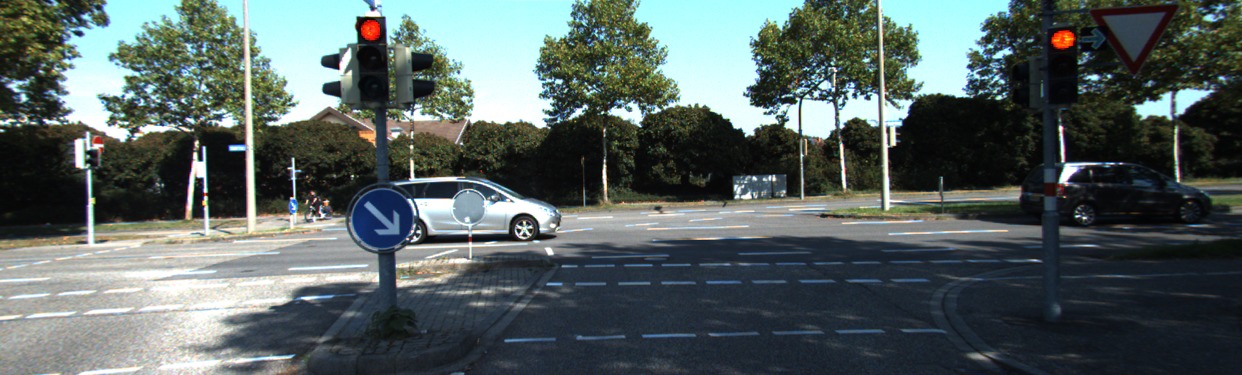}\\
\includegraphics[width=\columnwidth]{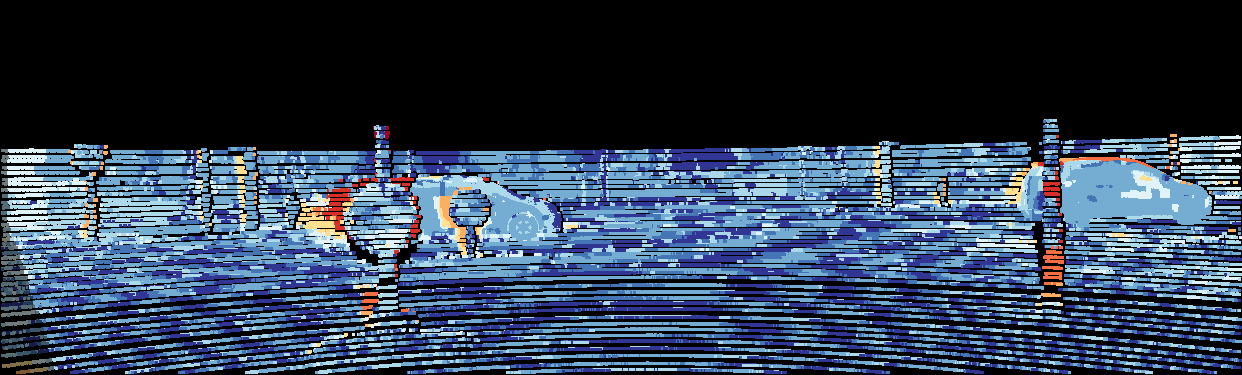}
\end{minipage}
\caption{Qualitative comparison on the test set of Kitti 2015. Cold colors = error smaller than 3 disparities, warm colors = error larger than 3 disparities.}
\label{fig:kittiQualitative}
\end{figure}

\begin{figure}
\centering
\begin{tabular}{c@{\hskip 3pt}c@{\hskip 3pt}c@{\hskip 3pt}c@{\hskip 6pt}c@{\hskip 3pt}c@{\hskip 3pt}c@{\hskip 3pt}c}
  \includegraphics[height=2.9cm]{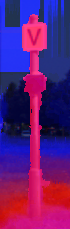} &
  \includegraphics[height=2.9cm]{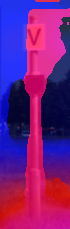} &
  \includegraphics[height=2.9cm]{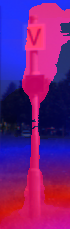} &
  \includegraphics[height=2.9cm]{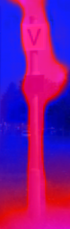} &
  \includegraphics[height=2.9cm]{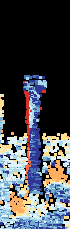} &
  \includegraphics[height=2.9cm]{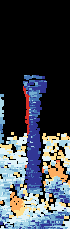} &
  \includegraphics[height=2.9cm]{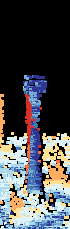} &  \includegraphics[height=2.9cm]{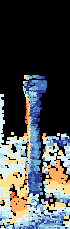} \\
  \vspace{-1pt}
  \small Ours & \small \cite{Zbontar2016} & \small \cite{Mayer_2016_CVPR} & \small \cite{Luo2016} & \small Ours & \small \cite{Zbontar2016} & \small \cite{Mayer_2016_CVPR} & \small \cite{Luo2016}
\end{tabular}
  \vspace{-2pt}
  \caption{Zoom-in comparison with state-of-the-art methods on a selected test image. Left images show an overlay of depth prediction and input image and right images show the corresponding error plots.}
  \label{fig:kittiZoom}
\end{figure}

\subsection{Benchmark Test Performance}
The complete evaluation of our submission on test images is available in the online suites of Middlebury~\cite{Scharstein2014} and Kitti 2015~\cite{Menze2015}.
The summary of this evaluation is presented in \cref{tbl:online}.
We want to stress that these results have been achieved without using any post-processing like occlusion detection and -inpainting or sub-pixel refinement.

We fine-tuned our best performing model (\cref{tab:evalAll}, CNN7+PW) for {\em half} sized images and used it for the Middlebury evaluation. \cref{tbl:online} shows the root mean squared (RMS) error metric and the {\em bad2} error metric for all test images. We achieve the lowest overall RMS error. Our {\em bad2} error is slightly worse compared to the other methods. These two results suggest our wrong counted disparities are just slightly beside. This behavior is shown in the error plot at the bottom in \cref{fig:middleburyTestQualitative}, where many small discretization artefacts are visible on slanted surfaces. Note that a sub-pixel refinement would remove most of this error. Additionally, we present an example where our algorithm achieves a very low error as in the majority of images.

For Kitti we use our best performing model (\cref{tab:evalAll}, CNN7+PW), including the $x$- and $y$-coordinates of the pixels as features. This is justified because the sky is always at the top of the image while the roads are always at the bottom for example.
The error plots for Kitti in \cref{fig:kittiQualitative} reveal that most of the incorrect predictions are in occluded areas. In \cref{fig:kittiZoom}  we show a qualitative comparison of magnified depth predictions of CNN-based methods on a Kitti test image. 
The depth overlays at the left side of the figure show how accurately the algorithms recover object boundaries and the images on the right side show the corresponding error plots provided by the evaluation system. Note, that very accurate predictions are partially treated as incorrect and how the competing methods tend to overfit to the fattened ground truth. Our approach works also very well in the upper third of the images, whereas the competing methods bleed out.


%% file: tex/conclusion.tex

\section{Conclusion}
We have proposed a fully trainable hybrid CNN+CRF model for stereo and its joint training procedure. 
Instead of relying on various post-processing procedures we designed a clean model without post-processing, where each part has its own responsibility. 
Therefore we gain interpretability of what is learned in each component of the model. 
This gives the insight that using a well defined model decreases the number of parameters significantly while still achieving a competitive performance.
We have shown that the joint training allows to learn unary costs as well as pairwise costs, while having the evidence that the increased generality always improves the performance.
Our newly proposed trainable pairwise terms allow to delineate object boundaries more accurately.
For the SSVM training we detailed the approximation of a subgradient and have shown that our training procedure works experimentally. 
For future work we plan to introduce an additional occlusion label to our model to further improve the performance in occluded areas. In addition, it will be interesting to investigate a continuous label space~\cite{moellenhoff-laude-cvpr16} to improve the performance of the model on slanted surfaces. 
\footnotetext[3]{With our CRF as postprocessing}

%% file: tex/supplementary.tex

\setcounter{figure}{0}
\setcounter{table}{0}
\counterwithin{figure}{section}
\counterwithin{table}{section}

\twocolumn[{%
 \centering
 \LARGE \mytitle \\ Supplementary Material \\[1.5em] 
 \normalsize
}]


\let\Contentsline\contentsline
\renewcommand\contentsline[3]{\Contentsline{#1}{#2}{}}
\makeatletter
\renewcommand{\@dotsep}{10000}
\makeatother

\section{Technical Details}\label{sec:suppl_detail}
\subsection{Approximate Subgradient}
Here we proof~\cref{P:unary-grad}, restated below for convenience.
\PUnaryGrad*
\begin{proof}
The loss upper bound~\eqref{LP-SSVM-bound} involves the minimum over $x^1$, $x^2$ as well as many minima inside the dynamic programming defining $\lambda$. A subgradient can be obtained by fixing particular minimizers in all these steps and evaluating the gradient of the resulting function.
It follows that a subgradient of the point-wise minimum of $(\bar f^1 + \lambda) (x^1) + (\bar f^2 - \lambda) (x^2)$ over $x^1,x^2$ can be chosen as $g =$
\begin{align}\label{subgrad-unary-suppl}
\nabla_{f_\V} (\bar f^1(\bar x^1) + \bar f^2(\bar x^2)) + \nabla_{\lambda} (\lambda (\bar x^1) - \lambda (\bar x^2)) J,
\end{align}
where $J_{i,j}(k,l)$ is a sub-Jacobian matching $\frac{d \lambda_j(l)}{d f_i(k)}$ for the directions $d f_\V$ such that $\lambda(f+df_\V)$ has the same minimizers inside dynamic programming as $\lambda(f)$.
\par
In the first part of the expression~\eqref{subgrad-unary-suppl}, the pairwise components and the loss $l(\bar x^1,x^*)$ do not depend on $f_i$ and may be dropped, leaving only $(\nabla_{f_\V} \sum_{j\in\V} f_{j}(\bar x^1_{j}))_i  = \delta(\bar x^1)_i$.
\par
Let $h$ denote the second expression in~\eqref{subgrad-unary-suppl}. Its component $h_i(k)$ expands as
\begin{subequations}
\begin{align}
h_i(k) = & \sum_{j\in\V}\sum_{l\in\L} \frac{\partial}{\partial \lambda_j(l)}(\lambda_{j} (\bar x^1_j) - \lambda_{j} (\bar x^2_j)) J_{ij}(k,l)\\
= & \sum_{j \in \Omega_{{\neq}}} \sum_{l\in\L}(\leftbb \bar x^1_j{=}l\rightbb - \leftbb \bar x^2_j=l\rightbb ) J_{ij}(k,l)\\
\label{neglected-sum}
= & \sum_{j \in \Omega_{{\neq}}} (J_{ij}(k,x^1_j) - J_{ij}(k,x^2_j)).
\end{align}
\end{subequations}
\qed
\end{proof}
Our intuition to neglect the sum~\eqref{neglected-sum} is as follows. We expect that variation of $f_{i}$ for a pixel $i$ far enough from $j \in \Omega_{\neq}$ will not have a significant effect on $\lambda_j$ and thus $J_{ij}$ will be small over $\Omega_{\neq}$. 

\section{Training insights}
We train our full joint model gradually as explained in \cref{sec:training} in the main paper. To give more insights on how the joint training evolves until we get our final parameters, we show a training plot in \cref{fig:training-validation}. 
This plot shows the evolution of the {\em bad4} error on the Middlebury dataset for our 7-layer model. We can identify three key steps during the training procedure. (A) shows the training of our \unarycnn using ML~\cref{subsec:training_unary}. In (B) we add the CRF with contrast-sensitive weights with an optimal choice of parameters $(\alpha,\beta,P_1,P_2)$. Finally, in (C) we jointly optimization the complete model~\cref{subsec:training_joint,sec:jointTraining}. Observe that the gap between training and validation errors is significantly smaller in (C).

\begin{figure}[ht]
\centering
\input{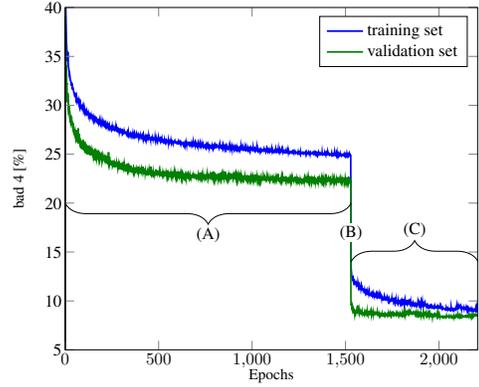}
\caption{
Performance \wrt the real objective 
for key complexity steps of our model during training.
\label{fig:training-validation}
}
\end{figure}
%
\section{Additional Experiments}\label{sec:suppl_exp}

\subsection{Timing}
In~\cref{tbl:timing} we report
the runtime of individual components of our method for different image sizes and number of labels (=disparties). All experiments are carried out on a Linux PC with a Intel Core i7-5820K CPU with 3.30GHz and a NVidia GTX TitanX using CUDA 8.0. For Kitti 2015, the image size is $1242\times 375$. For Middlebury V3 we selected the {\em Jadeplant} data set with {\em half} resolution, leading to an image size of $1318\times 994$. We observe that with a constant number of layers in the Unary CNN and disparity range, the runtime depends linearly on the number of pixels in the input images. Correlation and CRF layer also depend on the number of estimated disparities, where we report numbers using 128 and 256 disparities.

\begin{table}[tb]
  \centering
\setlength{\tabcolsep}{3pt}
\small
\begin{tabular}{ll d{2} d{2} d{2}}
  \toprule
  \textbf{Component} & \textbf{\# Disp.}& \multicolumn{1}{c}{\textbf{Kitti 2015}} & \multicolumn{1}{c}{\textbf{Middlebury}}&\multicolumn{1}{c}{\textbf{Real-Time}}\\
   & &\multicolumn{1}{c}{\unit[0.4]{MP}} & \multicolumn{1}{c}{\unit[1.3]{MP}} & \multicolumn{1}{c}{\unit[0.3]{MP}}\\
  \midrule
  Input processing &     &   7.58 & 6.40    & 6.02 \\
  Pairwise CNN     &     &  21.12 & 59.46   & 13.75\\
  Unary CNN        &     & 262.48 & 664.19  & 62.54\\
  Correlation      & 128 & 154.86 & 437.02  & 46.70\\
  Correlation      & 256 & 286.87 & 802.86  & -\\
  CRF              & 128 & 309.48 & 883.57  & 155.85\\
  CRF              & 256 & 605.35 & 1739.34 & -\\
  \midrule
  \textbf{Total}   & 128 & 755.52  & 2050.64 & 284.86\\
  \textbf{Total}   & 256 & 1183.40 & 3272.25 & -\\
  \bottomrule
\end{tabular}
\caption{Timing experiments for 7 layer CNN and 5 CRF iterations (3 layer and 4 iterations for {\bf Real-Time}). Runtimes in ms.}
\label{tbl:timing}
\end{table}

\subsection{Sublabel Enhancement}
\begin{figure}[ht]
\centering
\includegraphics[width=\columnwidth]{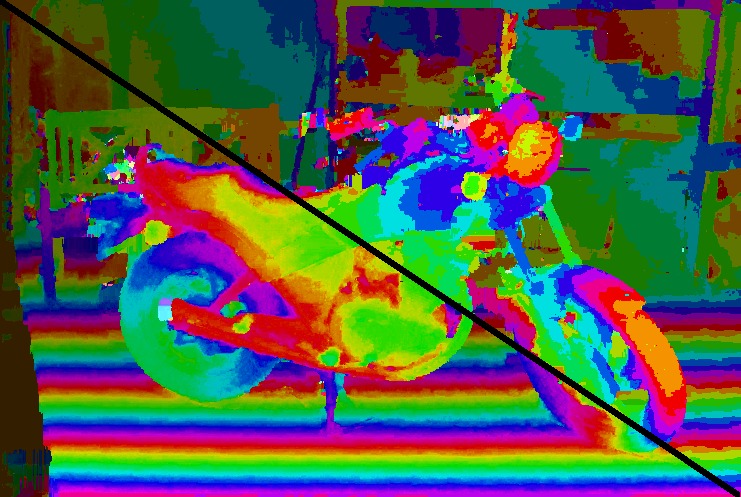}
\caption{Qualitative comparison on {\em Motorcycle} of discrete (upper-right) and sublabel enhanced (bottom-left) solution. Note how smooth the transitions are in the sublabel enhanced region (\eg at the floor or the rear wheel).}
\label{fig:middleburyRefinement}
\end{figure}

A drawback of our CRF method based on dynamic programming is the discrete nature of the solution. For some benchmarks like Middlebury the discretization artifacts negatively influence the quantitative performance. Therefore, most related stereo methods perform some kind of sub-label refinement (\eg~\cite{Luo2016,Zbontar2016}). For the submission to online benchmarks we deliberately chose to {\em discard} any form of non-trainable post-processing. However, we performed additional experiments with fitting a quadratic function to the output cost volume of the CRF method around the discrete solution. The refined disparity is then given by
\begin{equation}
d_{se} = d + \frac{C(d-h) - C(d+h)}{2 (C(d+h) - 2C(d) + C(d-h))}
\label{eq:sublabelEnhancement}
\end{equation}
where $C(d)$ is the cost of disparity $d$.
 A qualitative experiment on the {\em Motorcycle} image of Middlebury stereo can be seen in \cref{fig:middleburyRefinement}. Quantitative experiments have been conducted on both Kitti 2015 and Middlebury and will be reported in the follow sections (columns {\bf w. ref.} in \cref{tbl:all_mb,tbl:all_kitti_2015}). Again, in the main paper and in the submitted images we always report the performance of the {\em discrete} solution in order to keep the method pure.

%

\subsection{Middlebury Stereo v3}
In this section we report a complete overview of all tested variants of our proposed hybrid CNN-CRF model on the stereo benchmark of Middlebury Stereo v3. We report the mean error (error metric {\em percent of non-occluded pixels with an error bigger 4 pixels}). All results are calculated on quarter resolution and upsampled to the original image size. We present the results in \cref{fig:mbQualitative,tbl:all_mb}. Note, how the quality increases when we add more parameters and therefore allow a more general model (visualized from left to right in \cref{fig:mbQualitative}. The last row shows the \textit{Vintage} image, where our model produces a rather high error. The reason for that lies in the (almost) completely untextured region in the top-left corner. Our full model is able to recover some disparities in this region, but not all. A very interesting byproduct visible in \cref{fig:mbQualitative} concerns our small 3-layer model. Visually, one can hardly see any difference to the deeper 7-layer model, when our models are full jointly trained. Hence, this small model is suited very well for a real-time application.
\par
Additionally, we compared to the performance of the model learned on Kitti, denoted Kitti-CNN in \cref{tbl:all_mb}. The performance is inferior, which means that the model trained on Kitti does not generalize well to Middlebury. Generalizing from Middlebury to Kitti, on the other hand is much better, as discussed in the next section.

\begin{figure*}[h!t]
  \includegraphics[width=0.5\textwidth]{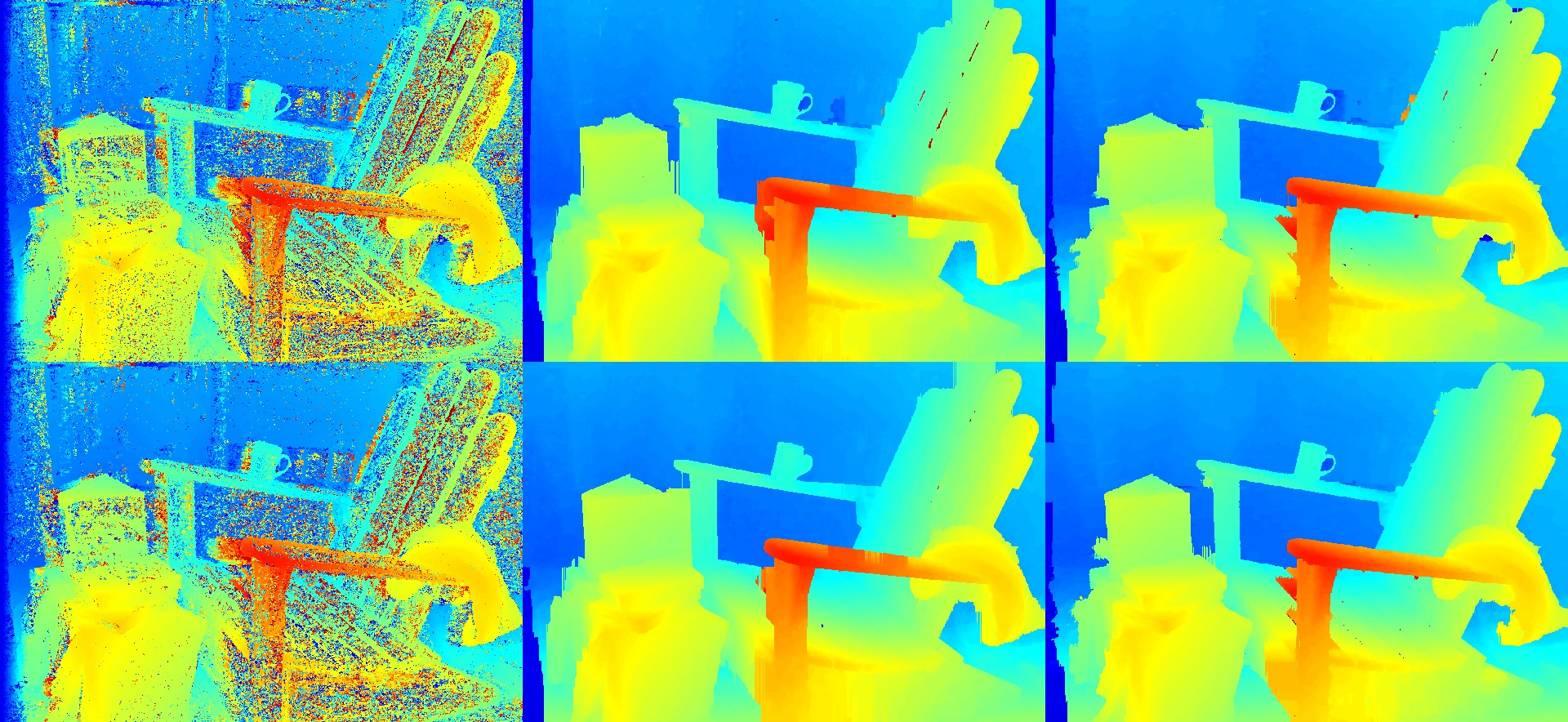}
  \includegraphics[width=0.5\textwidth]{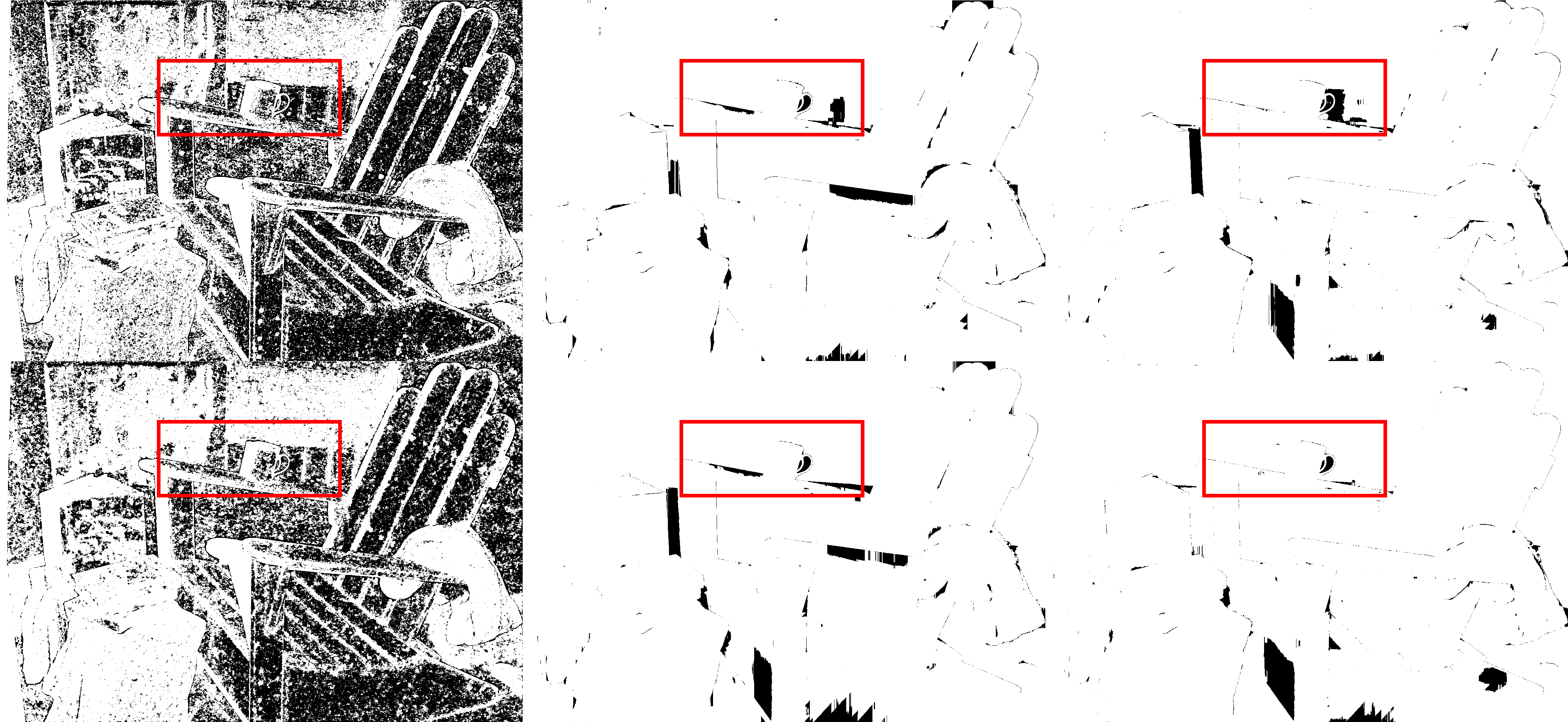}
  \includegraphics[width=0.5\textwidth]{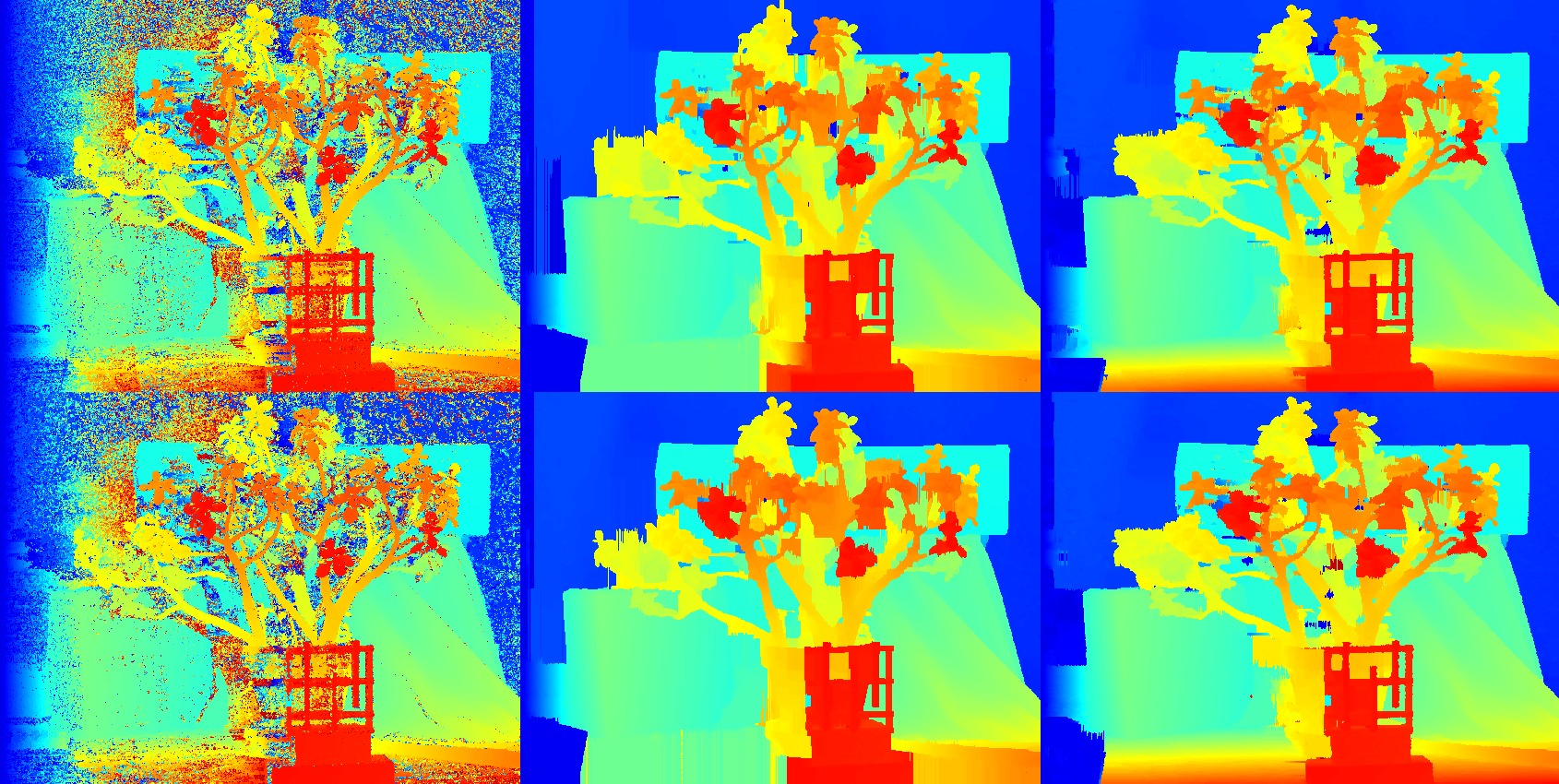}
  \includegraphics[width=0.5\textwidth]{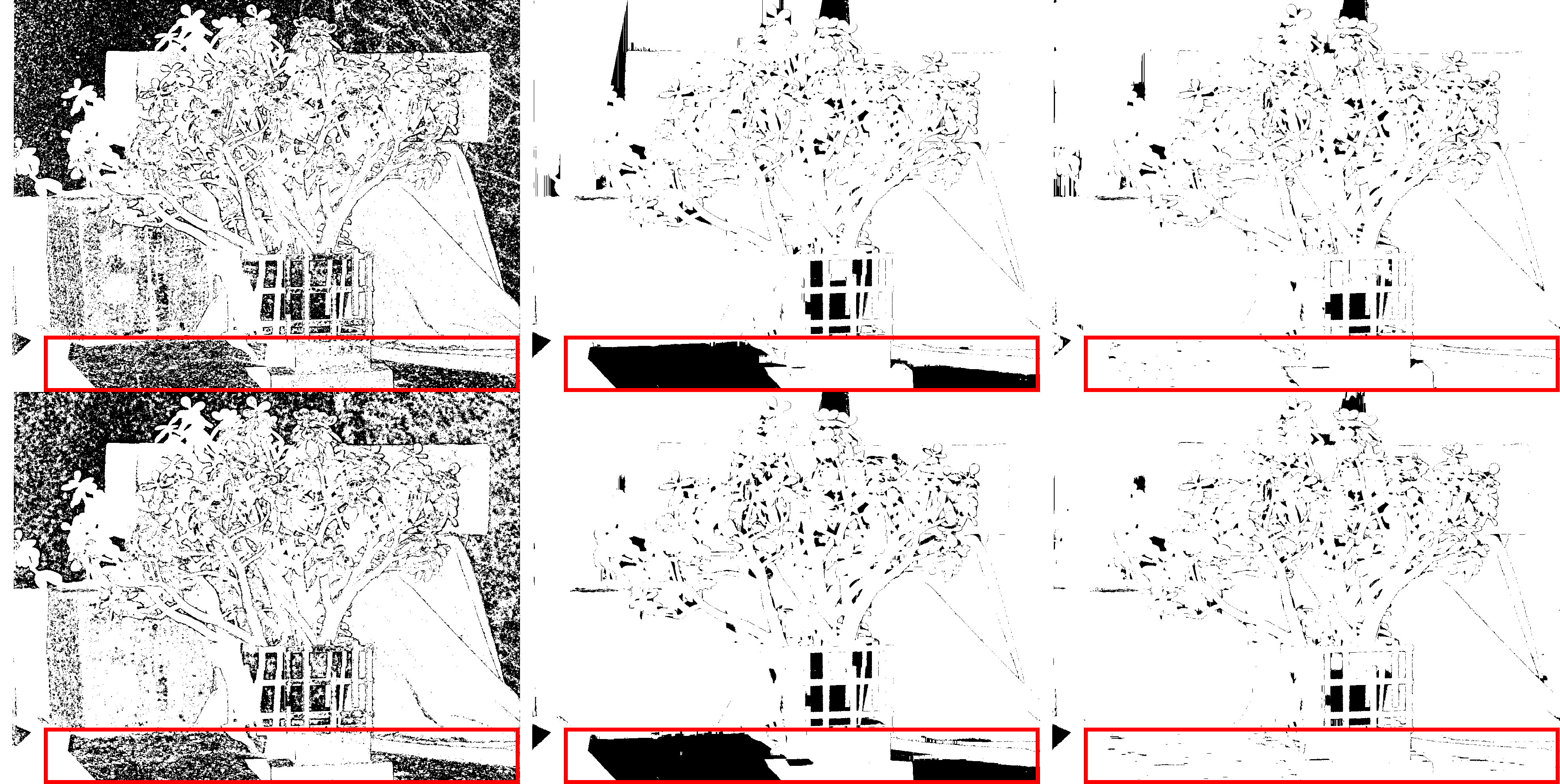}
  \includegraphics[width=0.5\textwidth]{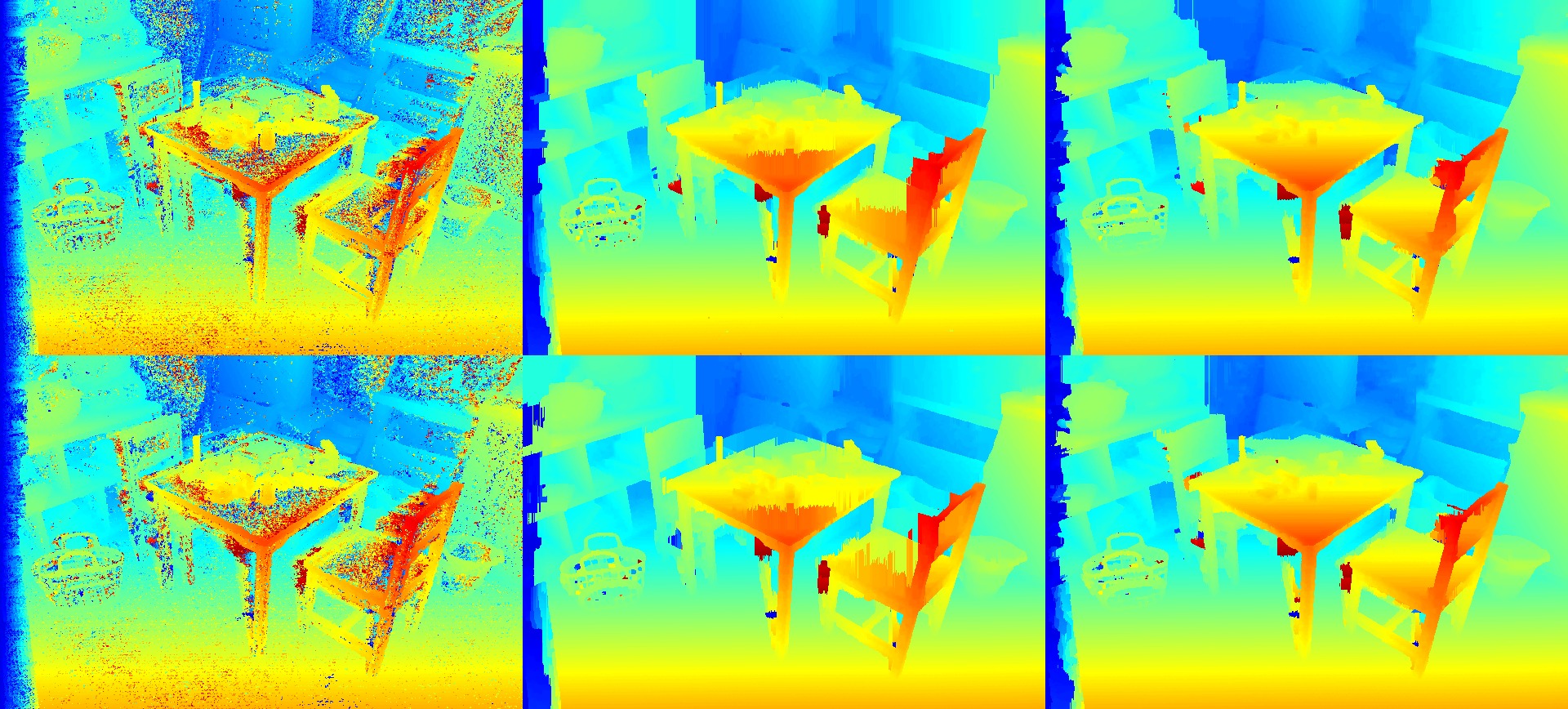}
  \includegraphics[width=0.5\textwidth]{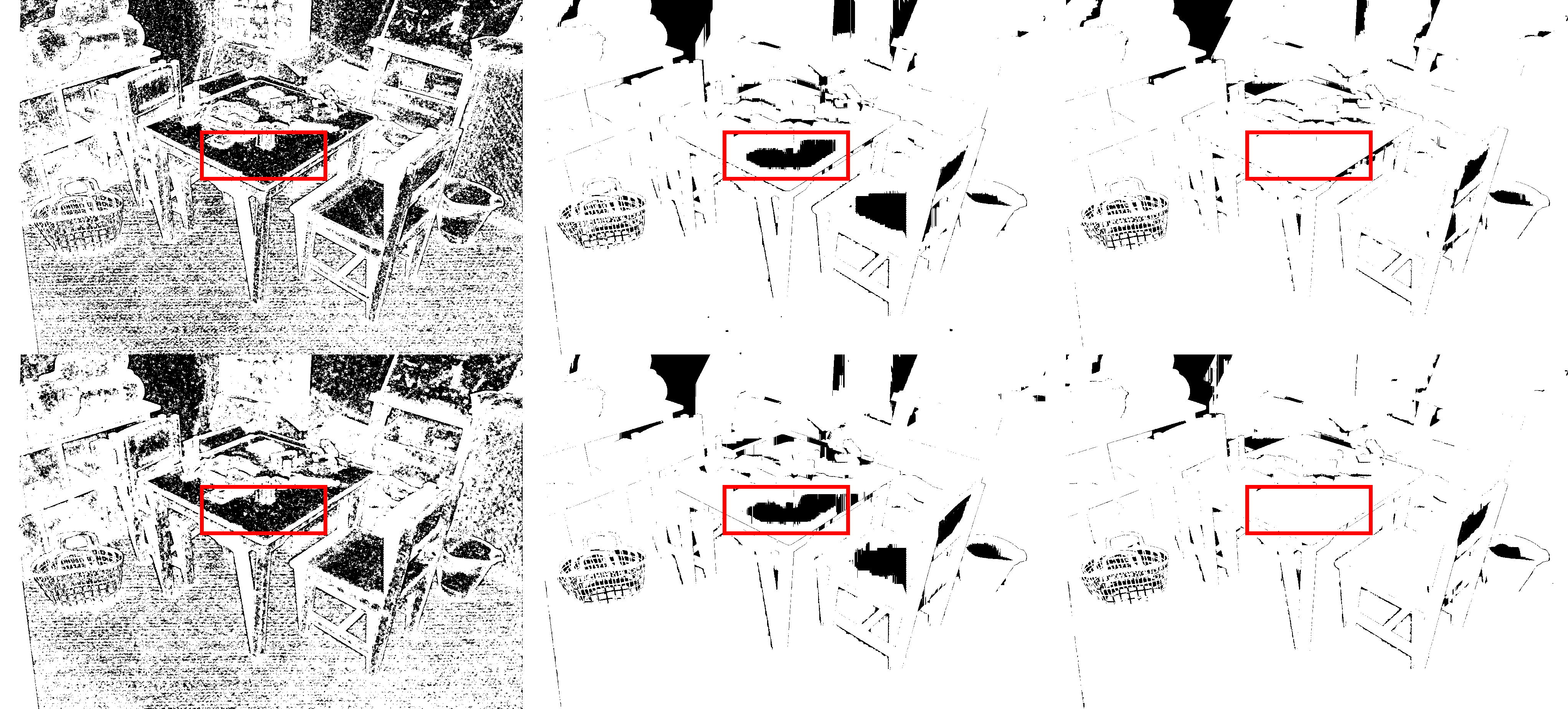}
  \includegraphics[
  width=0.5\textwidth]{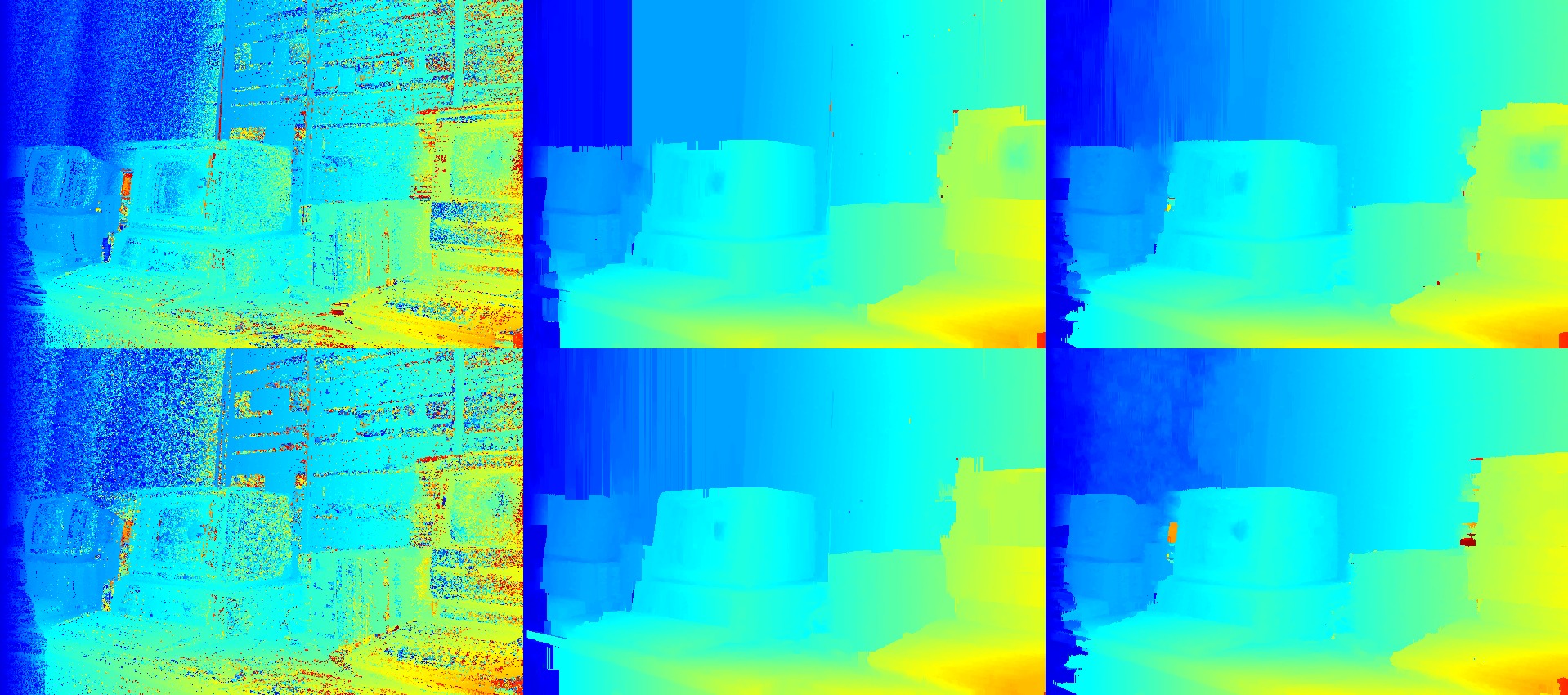}
  \includegraphics[width=0.5\textwidth]{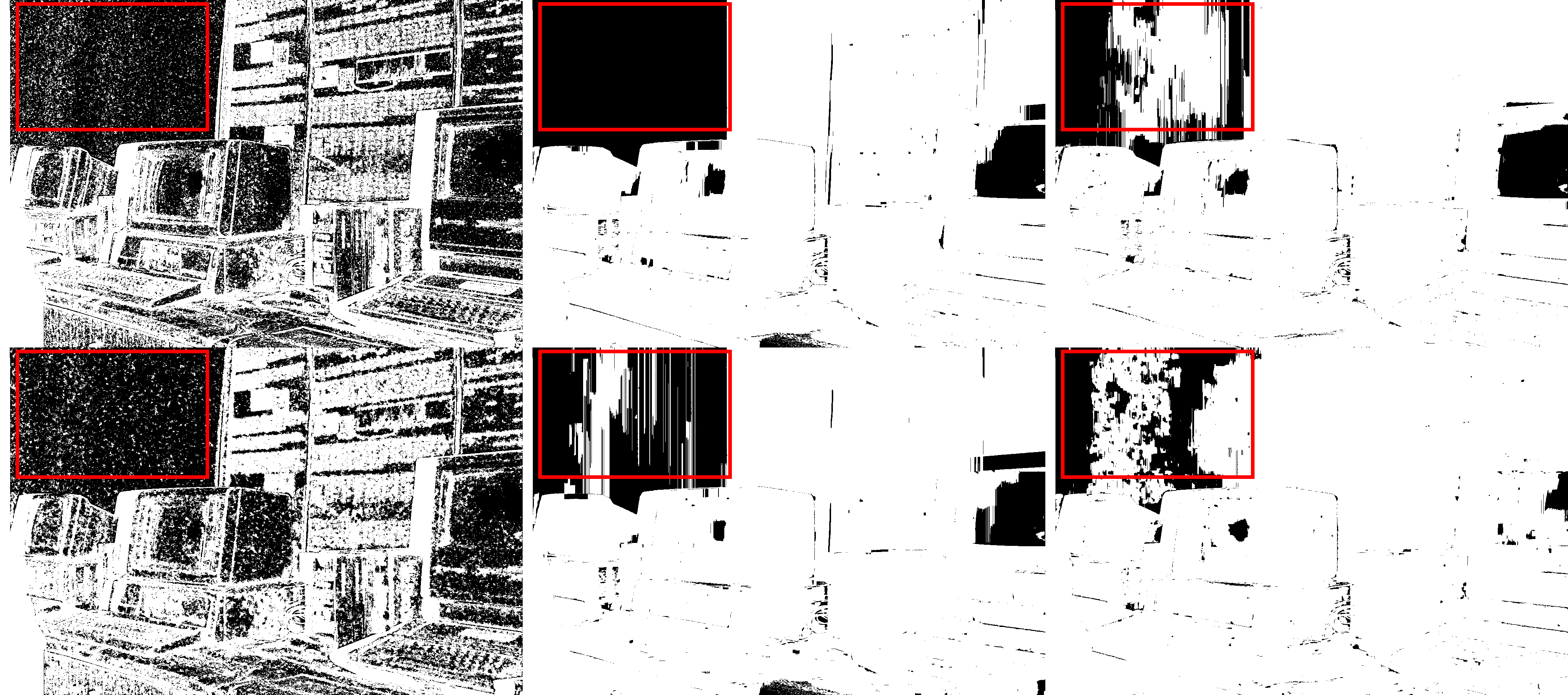}
  \caption{Qualitative comparison of our models on Middlebury. For each image, the first row shows our 3-layer model and the second row shows the result of our 7-layer model. The first column shows out \unarycnn with $\argmax$ desicion rule, the second column \textit{CNNx+CRF} and the third column shows the result of \textit{CNNx+CRF+Joint+PW}. The remaining columns show the respective error-plots for the different models, where white indicates correct and black indicates wrong disparities. The red boxes highlight differences between our models. Disparity maps are color-coded from blue (small disparities) to red (large disparities). }
  \label{fig:mbQualitative}
\end{figure*}

\begin{table}[ht]
  \centering
\setlength{\tabcolsep}{3pt}
\small
\begin{tabular}{lcc}
\toprule
\textbf{Method} & \textbf{w/o. ref.} & \textbf{w. ref.}\\
\midrule
CNN3 & 23.89 & - \\
CNN3+CRF & 11.18 & 10.50\\
CNN3 Joint & 9.48 & 8.75\\
CNN3 PW+Joint & 9.45 & 8.70\\
CNN7 & 18.58 & -\\
CNN7+CRF & 9.35 & 8.68\\
CNN7 Joint & 8.05 & 7.32 \\
CNN7 PW+Joint & 7.88 & 7.09\\
\midrule
Kitti-CNN & 15.22 & 14.43\\
\bottomrule
\end{tabular}
\caption{Comparison of differently trained models and their performance on the official training images of the Middlebury V3 stereo benchmark. The results are given in \% of pixels farther away than 4 disparities from the ground-truth on all pixels.}
\label{tbl:all_mb}
\end{table}

\subsection{Kitti 2015}
\label{ssec:supplkitti15}
In this section we report a complete overview of all tested variants of our proposed hybrid CNN-CRF model on the stereo benchmark of KITTI 2015. We report the mean error (official error metric {\em percent of pixel with an error bigger 3 pixels}) on the complete design set. \cref{tbl:all_kitti_2015} shows a performance overview of our models. In the last row of \cref{tbl:all_kitti_2015} we apply our best performing model on Middlebury to the Kitti design set. Interestingly, the performance decreases only by $\approx 1.5\%$ on all pixels. This experiment indicates, that our models generalize well to the scenes of the Kitti benchmark.

\begin{table}[ht]
  \centering
\setlength{\tabcolsep}{3pt}
\small
\begin{tabular}{lcccc}
\toprule
\multirow{2}{*}{\textbf{Method}} & \multicolumn{2}{c}{ \textbf{w/o. ref.}} & \multicolumn{2}{c}{\textbf{w. ref.}}\\
&                 \textbf{all} & \textbf{non occ.} & \textbf{all} & \textbf{non occ.} \\
\midrule
CNN3 & 29.58 & 28.38 & - & - \\
CNN3+CRF & 7.88 & 6.33 & 7.77 & 6.22\\
CNN3 Joint & 7.66 & 6.11 & 7.57 & 6.02\\
CNN3 PW+Joint & 6.25 & 4.75 & 6.14 & 4.65\\
CNN7 & 14.55 & 13.08 & - & - \\
CNN7+CRF & 5.85 & 4.79 & 5.76 & 4.70\\
CNN7 Joint & 5.98 & 4.60 & 5.89 & 4.50\\
CNN7 PW+Joint & 5.25 & 4.04 & 5.18 & 3.96\\
\midrule
\cite{Zbontar2016}+CRF & 6.10 & 4.45 & 5.74 & 4.08\\
\cite{Luo2016}+CRF & 5.89 & 4.31 & 5.81 & 4.21\\
\cite{Zbontar2016} & 15.02 & 13.56 & - & -\\
\cite{Luo2016} & 7.54 & 5.99 & - & -\\
\midrule
MB-CNN & 6.82 & 5.35 & 6.69 & 5.21\\
\bottomrule
\end{tabular}
\caption{Comparison of differently trained models and their performance on the design set images of the KITTI 2015 stereo benchmark. The results are given in \% of pixels farther away than 3 disparities from the ground-truth on all pixels.}
\label{tbl:all_kitti_2015}
\end{table}

Due to lack of space in the main paper, we could only show a few qualitative results of the submitted method. In \cref{fig:kittiQualitativeMore} we show additional results, more of which can be viewed online.

Looking at Kitti results in more detail, we observe that most of the errors happen in either occluded regions or due to a fattened ground-truth. Since we train edge-weights to courage label-jumps at strong object boundaries, our model yields very sharp results. It is these sharp edges in our solution which introduce some errors on the benchmark, even when our prediction is correct. \cref{fig:kittiDetail} shows some examples on the \textit{test} set (provided by the online submission system).

\begin{figure*}[ht]
\centering
  \includegraphics[height=2.9cm]{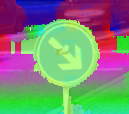}
  \includegraphics[height=2.9cm]{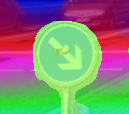}
  \includegraphics[height=2.9cm]{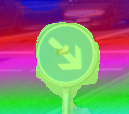}
  \includegraphics[height=2.9cm]{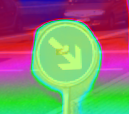} \\

  \includegraphics[height=2.9cm]{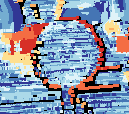}
  \includegraphics[height=2.9cm]{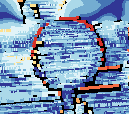}
  \includegraphics[height=2.9cm]{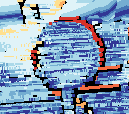}
  \includegraphics[height=2.9cm]{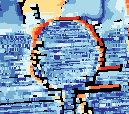} \\

  \includegraphics[height=2.97cm]{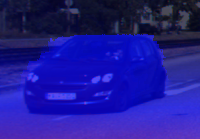}
  \includegraphics[height=2.97cm]{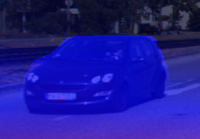}
  \includegraphics[height=2.97cm]{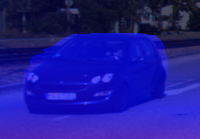}
  \includegraphics[height=2.97cm]{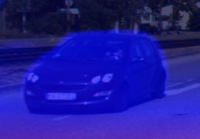} \\
  \includegraphics[height=2.97cm]{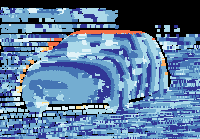}
  \includegraphics[height=2.97cm]{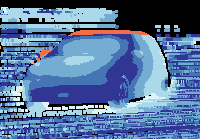}
  \includegraphics[height=2.97cm]{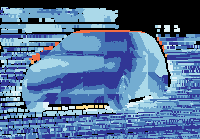}
  \includegraphics[height=2.97cm]{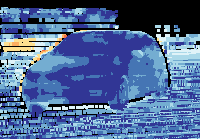}

  \caption{Error comparison on magnified parts of Kitti 2015 test images: The first and third row show the color-coded disparity map of \textit{Ours}, \textit{MC-CNN}, \textit{ContentCNN} and \textit{DispNetC}. The second and last row show the corresponding error-plots, where shades of blue mean correct and shades of orange mean wrong. Note, how our model accurately follows object boundaries, whereas all other approaches fatten the object. Nevertheless, in terms of correct or wrong we make more wrong predictions, because the ground-truth seems to be fattened as well.}
  \label{fig:kittiDetail}
\end{figure*}

%
%

\begin{figure*}[ht]
\centering
\includegraphics[width=0.66\textwidth]{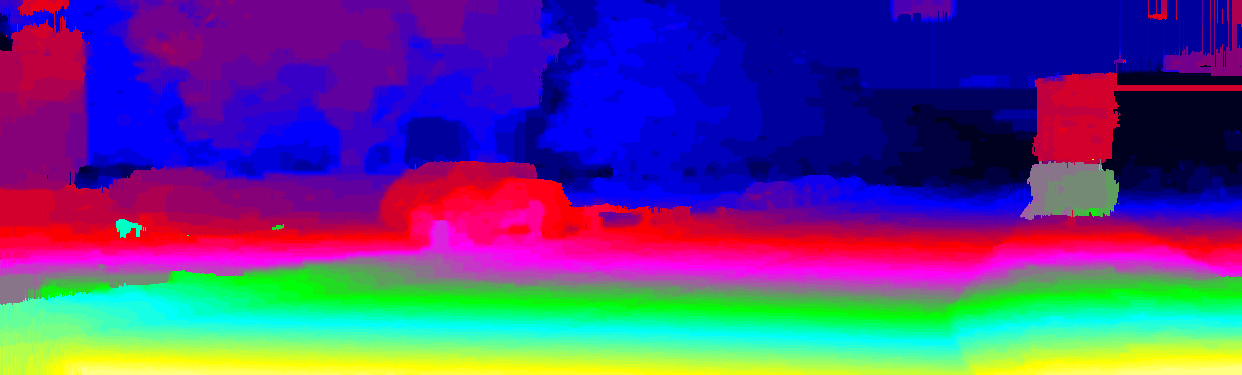}
\begin{minipage}[b][][c]{0.32\textwidth}
\includegraphics[width=\textwidth]{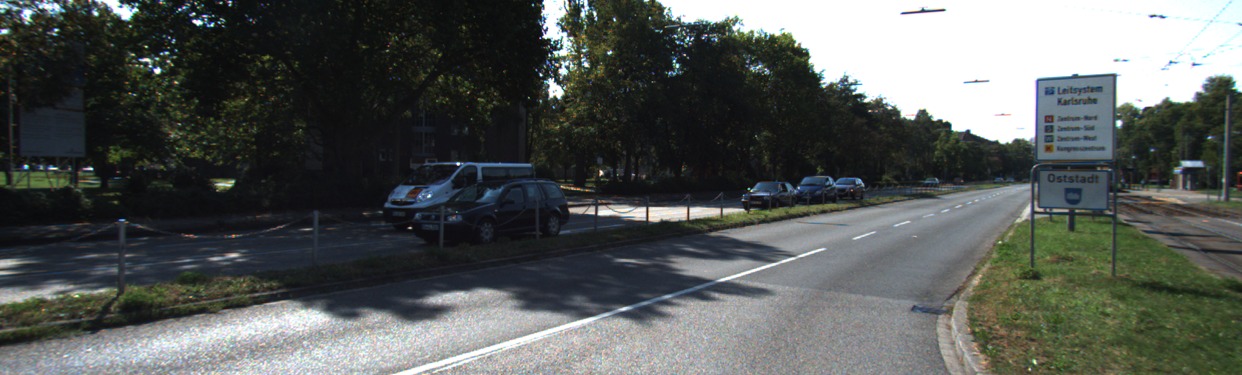}\\
\includegraphics[width=\textwidth]{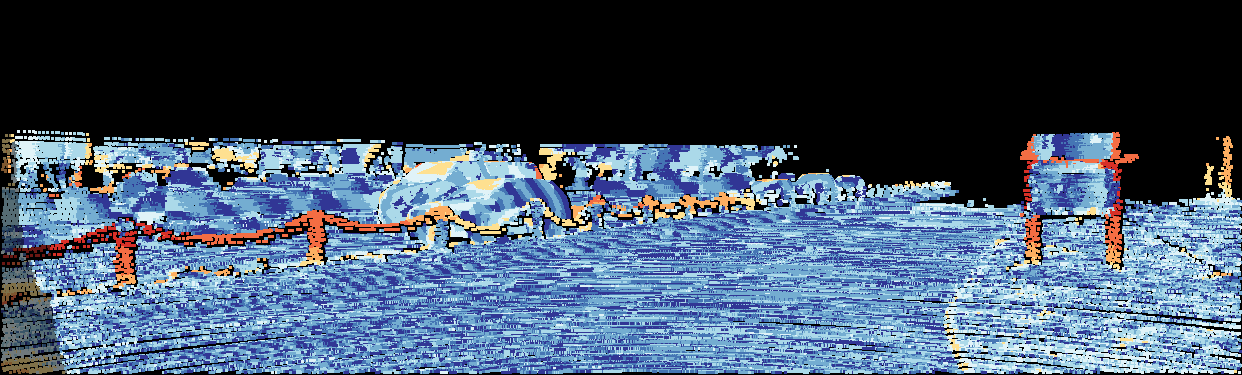}
\end{minipage}
\includegraphics[width=0.66\textwidth]{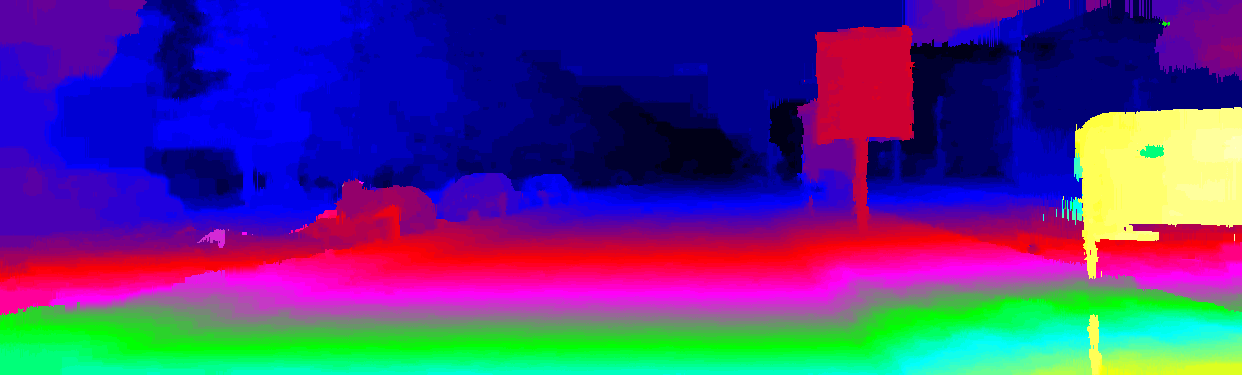}
\begin{minipage}[b][][c]{0.32\textwidth}
\includegraphics[width=\textwidth]{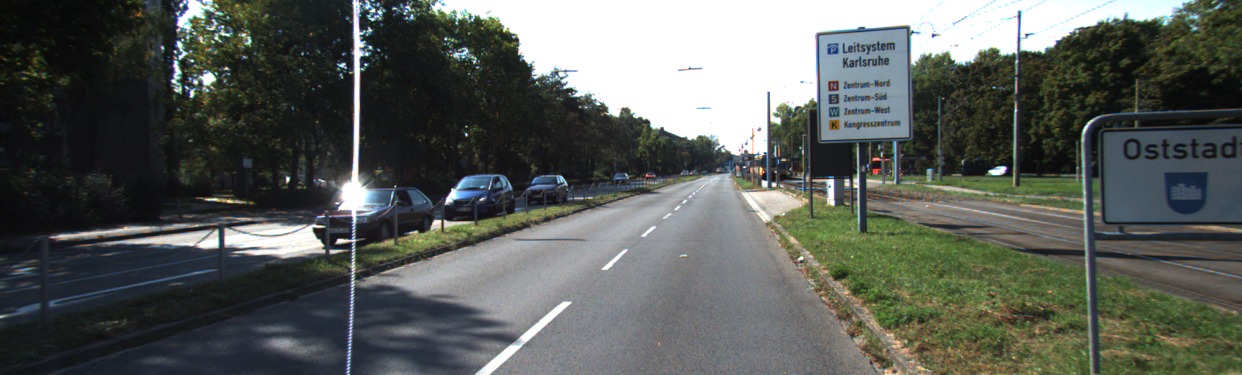}\\
\includegraphics[width=\textwidth]{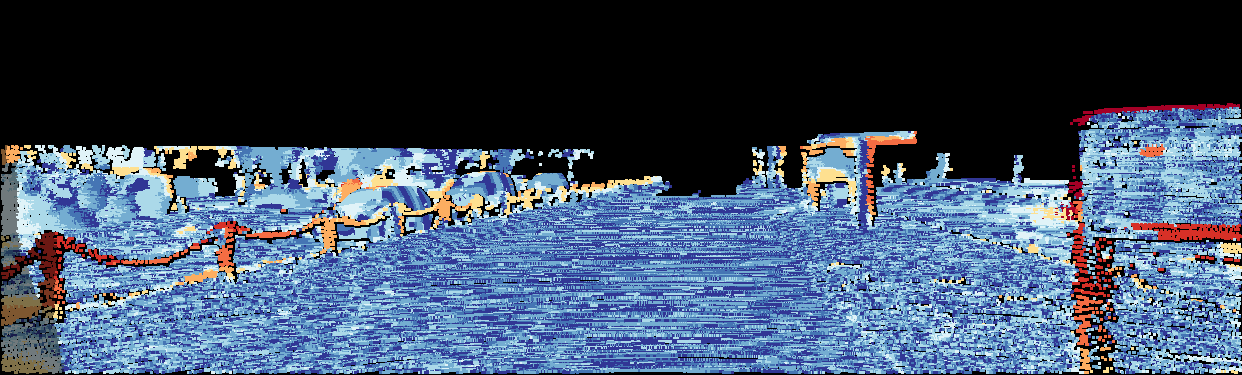}
\end{minipage}
\includegraphics[width=0.66\textwidth]{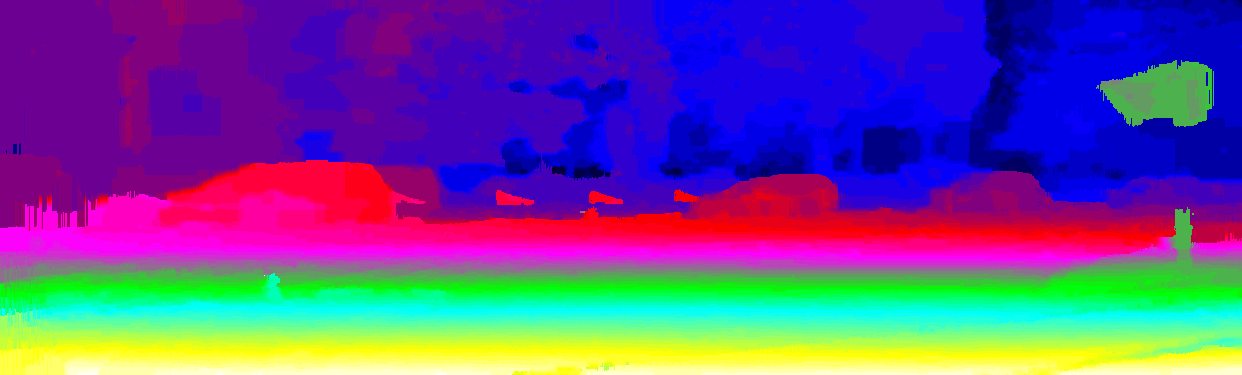}
\begin{minipage}[b][][c]{0.32\textwidth}
\includegraphics[width=\textwidth]{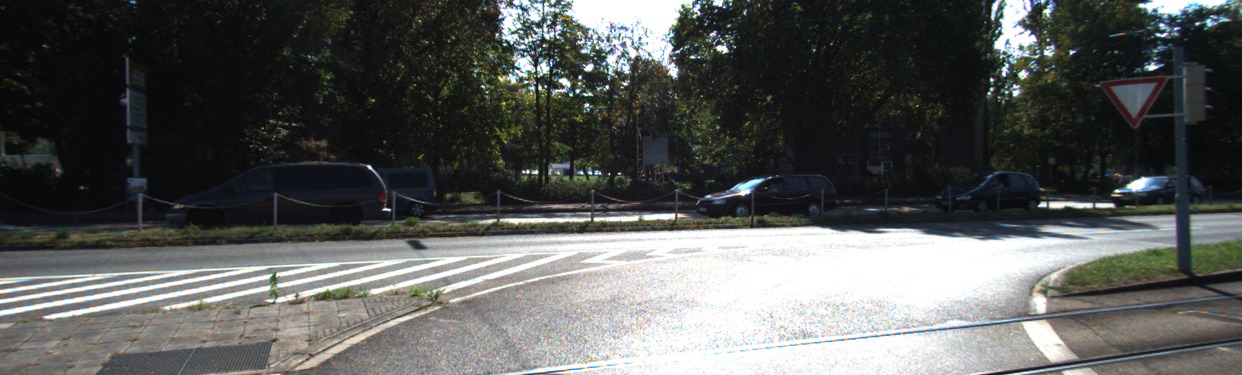}\\
\includegraphics[width=\textwidth]{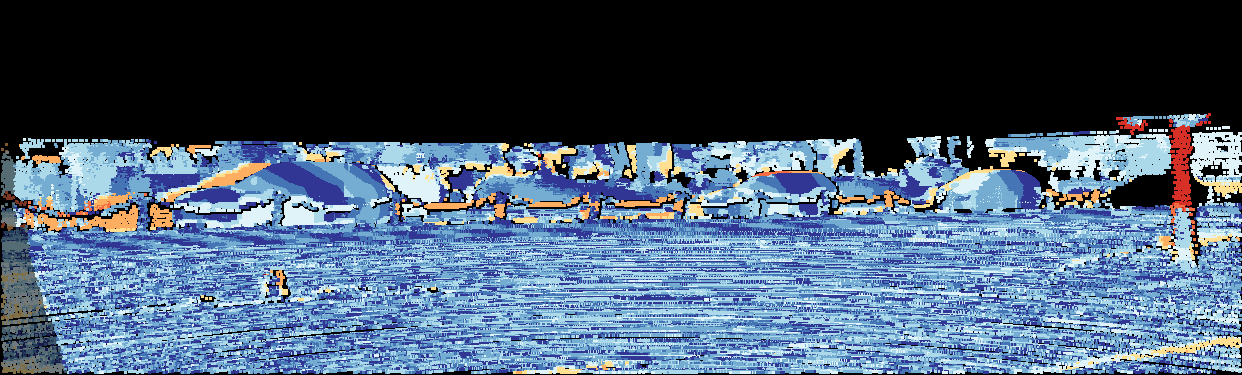}
\end{minipage}
\includegraphics[width=0.66\textwidth]{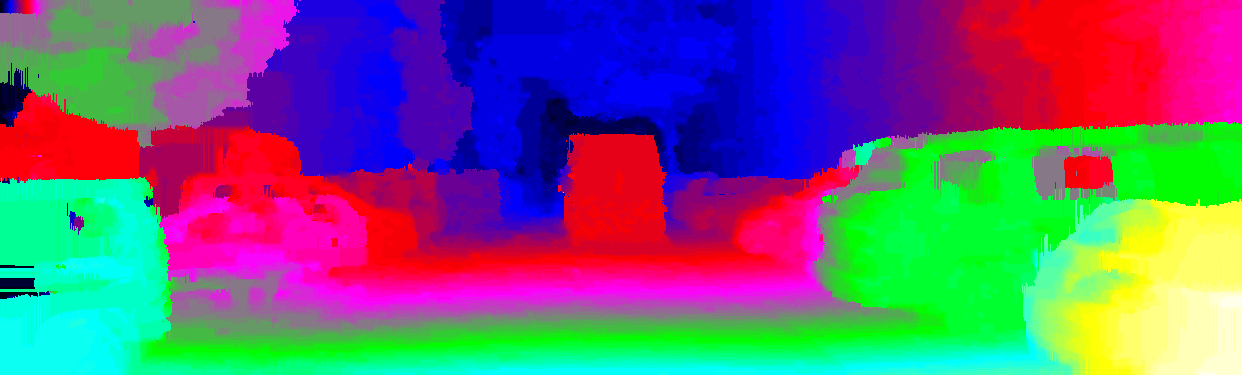}
\begin{minipage}[b][][c]{0.32\textwidth}
\includegraphics[width=\textwidth]{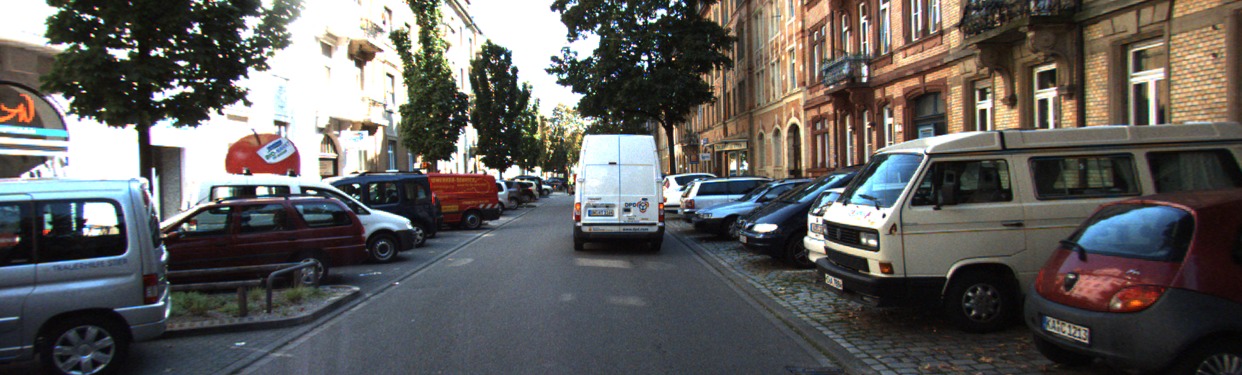}\\
\includegraphics[width=\textwidth]{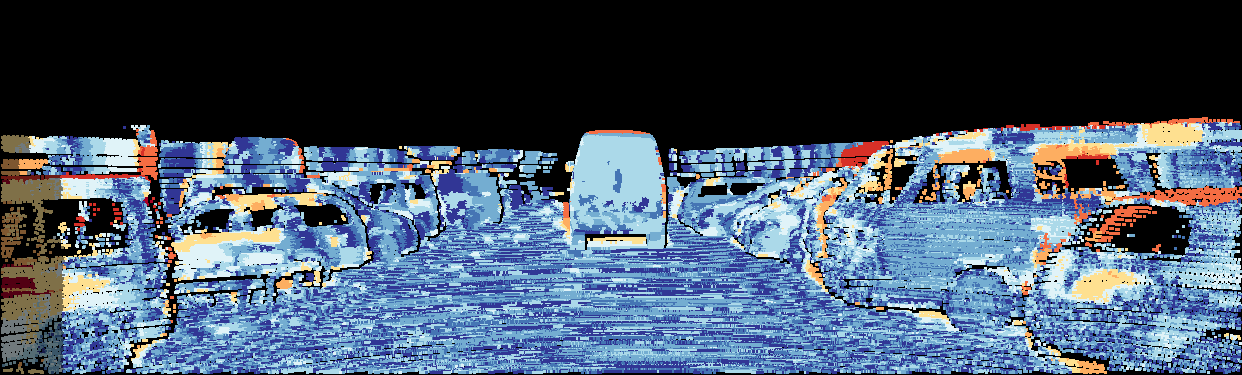}
\end{minipage}
\includegraphics[width=0.66\textwidth]{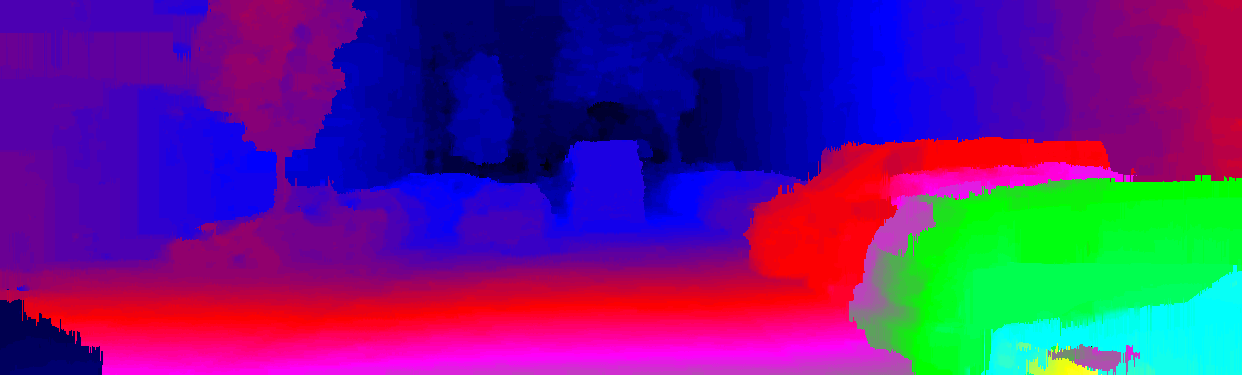}
\begin{minipage}[b][][c]{0.32\textwidth}
\includegraphics[width=\textwidth]{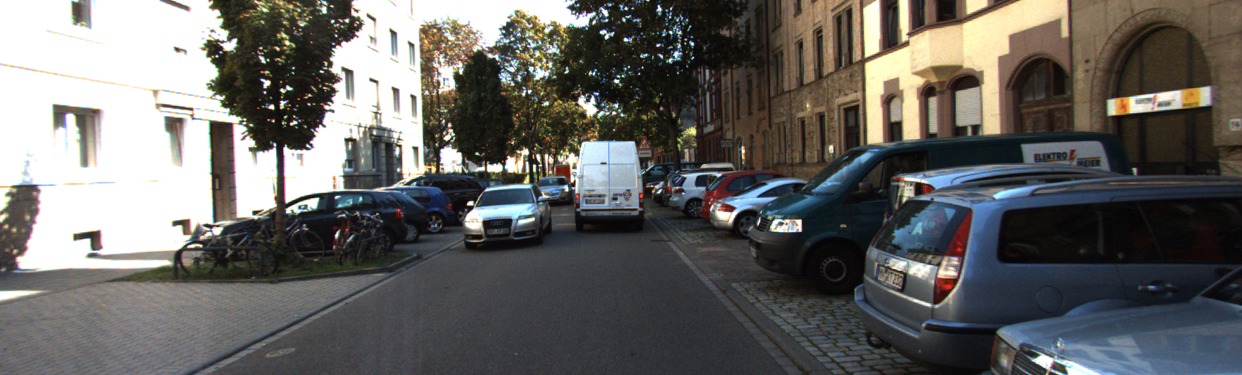}\\
\includegraphics[width=\textwidth]{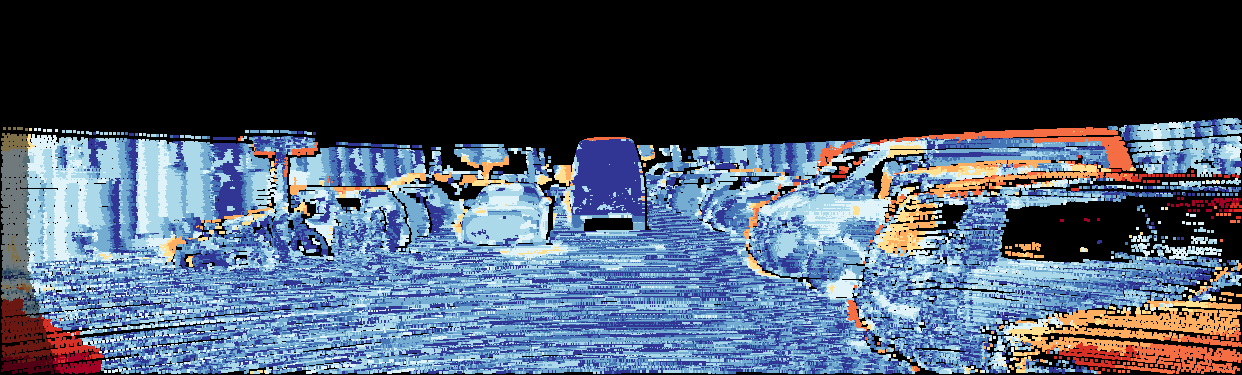}
\end{minipage}
\includegraphics[width=0.66\textwidth]{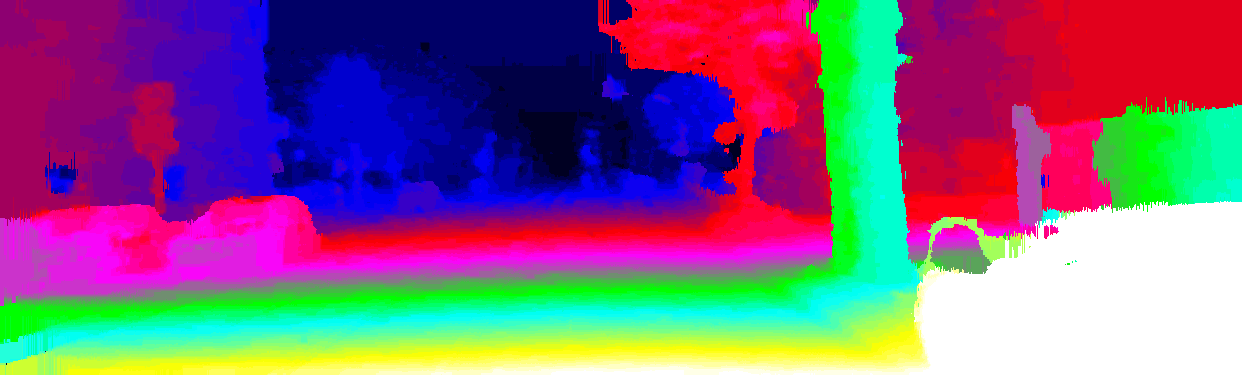}
\begin{minipage}[b][][c]{0.32\textwidth}
\includegraphics[width=\textwidth]{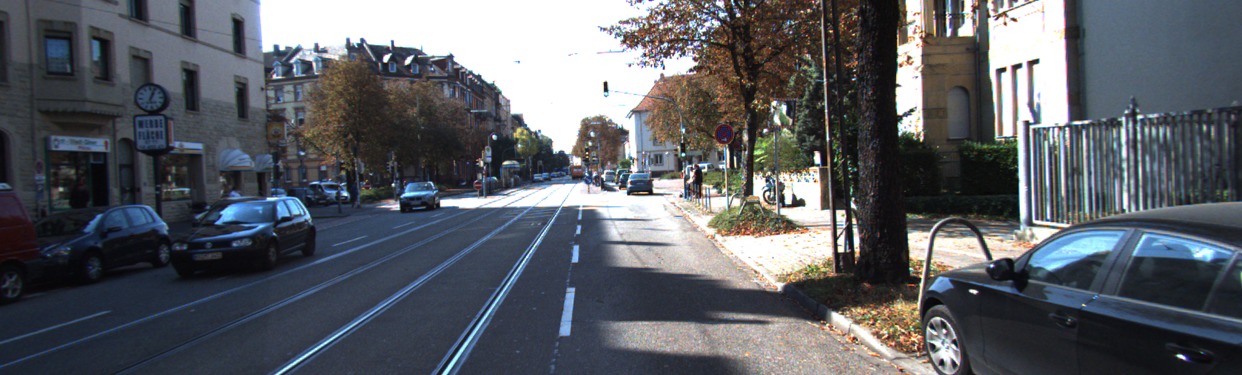}\\
\includegraphics[width=\textwidth]{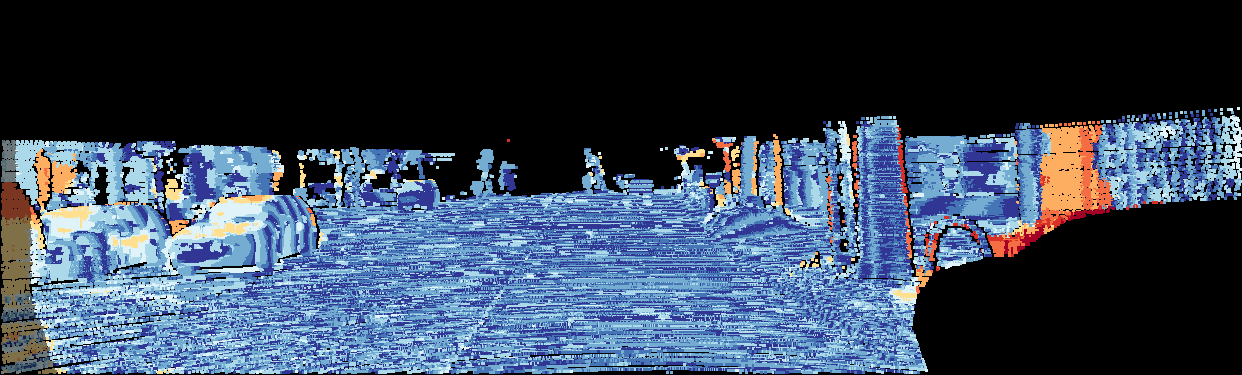}
\end{minipage}
\caption{Qualitative comparison on the test set of KITTI 2015.}
\label{fig:kittiQualitativeMore}
\end{figure*}


%

%% file: main.bbl
\begin{thebibliography}{}

\bibitem[\protect\astroncite{Alahari et~al.}{2010}]{Alahari10a}
Alahari, K., Russell, C., and Torr, P.~H.~S. (2010).
\newblock Efficient piecewise learning for conditional random fields.
\newblock In {\em Conference on Computer Vision and Pattern Recognition}.

\bibitem[\protect\astroncite{Bailer et~al.}{2016}]{BailerVS16}
Bailer, C., Varanasi, K., and Stricker, D. (2016).
\newblock {CNN} based patch matching for optical flow with thresholded hinge
  loss.
\newblock {\em CoRR}, abs/1607.08064.

\bibitem[\protect\astroncite{Barron and Poole}{2016}]{Barron2016}
Barron, J.~T. and Poole, B. (2016).
\newblock The fast bilateral solver.
\newblock In {\em European Conference on Computer Vision}.

\bibitem[\protect\astroncite{Bergstra et~al.}{2010}]{Bergstra2010}
Bergstra, J., Breuleux, O., Bastien, F., Lamblin, P., Pascanu, R., Desjardins,
  G., Turian, J., Warde-Farley, D., and Bengio, Y. (2010).
\newblock Theano: A cpu and gpu math expression compiler.
\newblock In {\em Python for Scientific Computing Conference}.

\bibitem[\protect\astroncite{Birchfield and Tomasi}{1998}]{Birchfield-98}
Birchfield, S. and Tomasi, C. (1998).
\newblock A pixel dissimilarity measure that is insensitive to image sampling.
\newblock {\em IEEE Trans. Pattern Anal. Mach. Intell.}, 20(4):401--406.

\bibitem[\protect\astroncite{Boykov and Jolly}{2000}]{Boykov00}
Boykov, Y. and Jolly, M.-P. (2000).
\newblock Interactive organ segmentation using graph cuts.
\newblock In {\em Medical Image Computing and Computer-Assisted Intervention},
  pages 276--286.

\bibitem[\protect\astroncite{Boykov and Jolly}{2001}]{Boykov01b}
Boykov, Y. and Jolly, M.-P. (2001).
\newblock Interactive graph cuts for optimal boundary \& region segmentation of
  objects in n-d images.
\newblock In {\em International Conference on Computer Vision}, pages 105--112.

\bibitem[\protect\astroncite{Bromley et~al.}{1993}]{Bromley1994}
Bromley, J., Bentz, J.~W., Bottou, L., Guyon, I., LeCun, Y., Moore, C.,
  S{\"a}ckinger, E., and Shah, R. (1993).
\newblock Signature verification using a siamese time delay neural network.
\newblock {\em International Journal of Pattern Recognition and Artificial
  Intelligence}, 7(04):669--688.

\bibitem[\protect\astroncite{Brown and S.}{2010}]{BHW10}
Brown, M.~Hua, G. and S., W. (2010).
\newblock Discriminative learning of local image descriptors.
\newblock {\em Transactions on Pattern Analysis and Machine Intelligence}.

\bibitem[\protect\astroncite{Chen et~al.}{}]{Chen2014a}
Chen, L.-C., Papandreou, G., Kokkinos, I., Murphy, K., and Yuille, A.~L.
\newblock Semantic image segmentation with deep convolutional nets and fully
  connected crfs.
\newblock {\em arXiv preprint arXiv:1412.7062}.

\bibitem[\protect\astroncite{Chen et~al.}{2015a}]{ChenSchwingICML2015}
Chen, L.-C., Schwing, A.~G., Yuille, A.~L., and Urtasun, R. (2015a).
\newblock {Learning Deep Structured Models}.
\newblock In {\em International Conference on Machine Learning}.

\bibitem[\protect\astroncite{Chen et~al.}{2015b}]{Chen2015}
Chen, Z., Sun, X., Wang, L., Yu, Y., and Huang, C. (2015b).
\newblock A deep visual correspondence embedding model for stereo matching
  costs.
\newblock In {\em International Conference on Computer Vision}, pages 972--980.

\bibitem[\protect\astroncite{Dosovitskiy et~al.}{2015}]{Dosovitskiy2015}
Dosovitskiy, A., Fischery, P., Ilg, E., Häusser, P., Hazirbas, C., Golkov, V.,
  v.~d. Smagt, P., Cremers, D., and Brox, T. (2015).
\newblock Flownet: Learning optical flow with convolutional networks.
\newblock In {\em International Conference on Computer Vision}, pages
  2758--2766.

\bibitem[\protect\astroncite{Facciolo et~al.}{2015}]{Facciolo-15}
Facciolo, G., de~Franchis, C., and Meinhardt, E. (2015).
\newblock {MGM}: A significantly more global matching for stereovision.
\newblock In {\em British Machine Vision Conference}.

\bibitem[\protect\astroncite{Franc and Laskov}{2011}]{Franc-Laskov-11}
Franc, V. and Laskov, P. (2011).
\newblock Learning maximal margin markov networks via tractable convex
  optimization.
\newblock {\em Control Systems and Computers}, pages 25--34.

\bibitem[\protect\astroncite{Hirschm{\"u}ller}{2005}]{Hirschmueller2005}
Hirschm{\"u}ller, H. (2005).
\newblock Accurate and efficient stereo processing by semi-global matching and
  mutual information.
\newblock In {\em Conference on Computer Vision and Pattern Recognition},
  volume~2, pages 807--814. IEEE.

\bibitem[\protect\astroncite{Hirschm{\"u}ller}{2011}]{hirschmuller2011semi}
Hirschm{\"u}ller, H. (2011).
\newblock Semi-global matching-motivation, developments and applications.
\newblock {\em Photogrammetric Week}.

\bibitem[\protect\astroncite{Kirillov et~al.}{2015}]{KirillovSFZ0TR15}
Kirillov, A., Schlesinger, D., Forkel, W., Zelenin, A., Zheng, S., Torr, P.
  H.~S., and Rother, C. (2015).
\newblock Efficient likelihood learning of a generic {CNN-CRF} model for
  semantic segmentation.
\newblock {\em CoRR}, abs/1511.05067.

\bibitem[\protect\astroncite{Kolmogorov}{2006}]{Kolmogorov-06-convergent-pami}
Kolmogorov, V. (2006).
\newblock Convergent tree-reweighted message passing for energy minimization.
\newblock {\em Transactions on Pattern Analysis and Machine Intelligence},
  28(10).

\bibitem[\protect\astroncite{Kolmogorov and Zabih}{2006}]{Kolmogorov2006}
Kolmogorov, V. and Zabih, R. (2006).
\newblock {\em Graph Cut Algorithms for Binocular Stereo with Occlusions},
  pages 423--437.
\newblock Springer US, Boston, MA.

\bibitem[\protect\astroncite{Komodakis}{2011}]{Komodakis-11-train}
Komodakis, N. (2011).
\newblock Efficient training for pairwise or higher order {CRF}s via dual
  decomposition.
\newblock In {\em Conference on Computer Vision and Pattern Recognition}, pages
  1841--1848.

\bibitem[\protect\astroncite{Komodakis et~al.}{2007}]{Komodakis-subgradient}
Komodakis, N., Paragios, N., and Tziritas, G. (2007).
\newblock {MRF} optimization via dual decomposition: Message-passing revisited.
\newblock In {\em International Conference on Computer Vision}, pages 1--8.

\bibitem[\protect\astroncite{Kr{\"a}henb{\"u}hl and
  Koltun}{2012}]{Kraehenbuehl2012}
Kr{\"a}henb{\"u}hl, P. and Koltun, V. (2012).
\newblock Efficient inference in fully connected crfs with gaussian edge
  potentials.
\newblock In {\em Neuro Information Processing Systems}.

\bibitem[\protect\astroncite{Laude et~al.}{2016}]{Laude2016}
Laude, E., M{\"o}llenhoff, T., Moeller, M., Lellmann, J., and Cremers, D.
  (2016).
\newblock Sublabel-accurate convex relaxation of vectorial multilabel energies.
\newblock In Leibe, B., Matas, J., Sebe, N., and Welling, M., editors, {\em
  European Conference on Computer Vision}, pages 614--627, Cham. Springer
  International Publishing.

\bibitem[\protect\astroncite{Li et~al.}{2016}]{Li2016}
Li, M., Shekhovtsov, A., and Huber, D. (2016).
\newblock Complexity of discrete energy minimization problems.
\newblock In {\em European Conference on Computer Vision}, pages 834--852.

\bibitem[\protect\astroncite{Lin et~al.}{2015}]{LinSRH15}
Lin, G., Shen, C., Reid, I.~D., and van~den Hengel, A. (2015).
\newblock Efficient piecewise training of deep structured models for semantic
  segmentation.
\newblock {\em CoRR}, abs/1504.01013.

\bibitem[\protect\astroncite{Liu et~al.}{2015}]{Liu_2015_ICCV}
Liu, Z., Li, X., Luo, P., Loy, C.-C., and Tang, X. (2015).
\newblock Semantic image segmentation via deep parsing network.
\newblock In {\em International Conference on Computer Vision}.

\bibitem[\protect\astroncite{Luo et~al.}{2016}]{Luo2016}
Luo, W., Schwing, A., and Urtasun, R. (2016).
\newblock Efficient deep learning for stereo matching.
\newblock In {\em International Conference on Computer Vision and Pattern
  Recognition}.

\bibitem[\protect\astroncite{Mayer et~al.}{2016}]{Mayer_2016_CVPR}
Mayer, N., Ilg, E., Hausser, P., Fischer, P., Cremers, D., Dosovitskiy, A., and
  Brox, T. (2016).
\newblock A large dataset to train convolutional networks for disparity,
  optical flow, and scene flow estimation.
\newblock In {\em Conference on Computer Vision and Pattern Recognition}.

\bibitem[\protect\astroncite{Menze and Geiger}{2015}]{Menze2015}
Menze, M. and Geiger, A. (2015).
\newblock Object scene flow for autonomous vehicles.
\newblock In {\em Conference on Computer Vision and Pattern Recognition}.

\bibitem[\protect\astroncite{M\"ollenhoff
  et~al.}{2016}]{moellenhoff-laude-cvpr16}
M\"ollenhoff, T., Laude, E., Moeller, M., Lellmann, J., and Cremers, D. (2016).
\newblock Sublabel-accurate relaxation of nonconvex energies.
\newblock In {\em Computer Vision and Pattern Recognition (CVPR)}.

\bibitem[\protect\astroncite{Nowozin}{2013}]{Nowozin-13}
Nowozin, S. (2013).
\newblock Constructing composite likelihoods in general random fields.
\newblock In {\em ICML Workshop on Infering: Interactions between Inference and
  Learning}.

\bibitem[\protect\astroncite{Ochs et~al.}{2015}]{Ochs2015}
Ochs, P., Ranftl, R., Brox, T., and Pock, T. (2015).
\newblock Bilevel optimization with nonsmooth lower level problems.
\newblock In Aujol, J.-F., Nikolova, M., and Papadakis, N., editors, {\em
  Conference on Scale Space and Variational Methods in Computer Vision}, pages
  654--665, Cham. Springer International Publishing.

\bibitem[\protect\astroncite{{Ochs} et~al.}{2016}]{Ochs-16}
{Ochs}, P., {Ranftl}, R., {Brox}, T., and {Pock}, T. (2016).
\newblock {Techniques for Gradient Based Bilevel Optimization with Nonsmooth
  Lower Level Problems}.
\newblock {\em ArXiv e-prints}.

\bibitem[\protect\astroncite{Pal et~al.}{2012}]{Christopher-12-learnig}
Pal, C.~J., Weinman, J.~J., Tran, L.~C., and Scharstein, D. (2012).
\newblock On learning conditional random fields for stereo - exploring model
  structures and approximate inference.
\newblock {\em International Journal of Computer Vision}, 99(3):319--337.

\bibitem[\protect\astroncite{Psota et~al.}{2015}]{Psota_2015_ICCV}
Psota, E.~T., Kowalczuk, J., Mittek, M., and Perez, L.~C. (2015).
\newblock Map disparity estimation using hidden {M}arkov trees.
\newblock In {\em ICCV}.

\bibitem[\protect\astroncite{Ranftl et~al.}{2014}]{ranftl2014non}
Ranftl, R., Bredies, K., and Pock, T. (2014).
\newblock Non-local total generalized variation for optical flow estimation.
\newblock In {\em European Conference on Computer Vision}, pages 439--454.
  Springer International Publishing.

\bibitem[\protect\astroncite{Ranftl and Pock}{2014}]{Ranftl-14}
Ranftl, R. and Pock, T. (2014).
\newblock A deep variational model for image segmentation.
\newblock In Jiang, X., Hornegger, J., and Koch, R., editors, {\em German
  Conference on Pattern Recognition}, pages 107--118, Cham. Springer
  International Publishing.

\bibitem[\protect\astroncite{Scharstein}{2007}]{Scharstein07learnin}
Scharstein, D. (2007).
\newblock Learning conditional random fields for stereo.
\newblock In {\em Conference on Computer Vision and Pattern Recognition}.

\bibitem[\protect\astroncite{Scharstein et~al.}{2014}]{Scharstein2014}
Scharstein, D., Hirschm¸ller, H., Kitajima, Y., Krathwohl, G., Nesic, N.,
  Wang, X., and Westling, P. (2014).
\newblock High-resolution stereo datasets with subpixel-accurate ground truth.
\newblock In {\em German Conference on Pattern Recognition}.

\bibitem[\protect\astroncite{Scharstein and Szeliski}{2002}]{Scharstein2002}
Scharstein, D. and Szeliski, R. (2002).
\newblock A taxonomy and evaluation of dense two-frame stereo correspondence
  algorithms.
\newblock {\em International journal of computer vision}, 47(1-3):7--42.

\bibitem[\protect\astroncite{Schwing and Urtasun}{2015}]{Schwing-15}
Schwing, A.~G. and Urtasun, R. (2015).
\newblock Fully connected deep structured networks.
\newblock {\em CoRR}, abs/1503.02351.

\bibitem[\protect\astroncite{Seki and Pollefeys}{2016}]{Seki2016}
Seki, A. and Pollefeys, M. (2016).
\newblock Patch based confidence prediction for dense disparity map.
\newblock In {\em British Machine Vision Conference (BMVC)}, volume~10.

\bibitem[\protect\astroncite{Shekhovtsov et~al.}{2016}]{Discrete-Continuous-16}
Shekhovtsov, A., Reinbacher, C., Graber, G., and Pock, T. (2016).
\newblock Solving dense image matching in real-time using discrete-continuous
  optimization.
\newblock In {\em Computer Vision Winter Workshop}, page~13.

\bibitem[\protect\astroncite{Simo-Serra et~al.}{2015}]{SimoSerraICCV2015}
Simo-Serra, E., Trulls, E., Ferraz, L., Kokkinos, I., Fua, P., and
  Moreno-Noguer, F. (2015).
\newblock {Discriminative Learning of Deep Convolutional Feature Point
  Descriptors}.
\newblock In {\em International Conference on Computer Vision}.

\bibitem[\protect\astroncite{Taskar et~al.}{2003}]{Taskar03max-marginmarkov}
Taskar, B., Guestrin, C., and Koller, D. (2003).
\newblock Max-margin markov networks.
\newblock MIT Press.

\bibitem[\protect\astroncite{Tompson et~al.}{2014}]{Tompson-14}
Tompson, J.~J., Jain, A., Lecun, Y., and Bregler, C. (2014).
\newblock Joint training of a convolutional network and a graphical model for
  human pose estimation.
\newblock In Ghahramani, Z., Welling, M., Cortes, C., Lawrence, N., and
  Weinberger, K., editors, {\em Advances in Neural Information Processing
  Systems 27}, pages 1799--1807. Curran Associates, Inc.

\bibitem[\protect\astroncite{Tsochantaridis et~al.}{2005}]{Tsochantaridis-2005}
Tsochantaridis, I., Joachims, T., Hofmann, T., and Altun, Y. (2005).
\newblock Large margin methods for structured and interdependent output
  variables.
\newblock {\em J. Mach. Learn. Res.}, 6:1453--1484.

\bibitem[\protect\astroncite{Werner}{2007}]{Werner-PAMI07}
Werner, T. (2007).
\newblock A linear programming approach to max-sum problem: {A} review.
\newblock {\em Transactions on Pattern Analysis and Machine Intelligence},
  29(7).

\bibitem[\protect\astroncite{Woodford et~al.}{2009}]{Woodford-08}
Woodford, O., Torr, P., Reid, I., and Fitzgibbon, A. (2009).
\newblock Global stereo reconstruction under second-order smoothness priors.
\newblock {\em Transactions on Pattern Analysis and Machine Intelligence},
  31(12).

\bibitem[\protect\astroncite{Zabih and Woodfill}{1994}]{Zabih1994}
Zabih, R. and Woodfill, J. (1994).
\newblock Non-parametric local transforms for computing visual correspondence.
\newblock In {\em European Conference on Computer Vision}, volume 801.

\bibitem[\protect\astroncite{Zagoruyko and Komodakis}{2015}]{Zagoruyko-15}
Zagoruyko, S. and Komodakis, N. (2015).
\newblock Learning to compare image patches via convolutional neural networks.
\newblock In {\em Conference on Computer Vision and Pattern Recognition}.

\bibitem[\protect\astroncite{{\v{Z}}bontar and LeCun}{2015a}]{Zbontar2015a}
{\v{Z}}bontar, J. and LeCun, Y. (2015a).
\newblock Computing the stereo matching cost with a convolutional neural
  network.
\newblock In {\em Conference on Computer Vision and Pattern Recognition}, pages
  1592--1599.

\bibitem[\protect\astroncite{{\v{Z}}bontar and LeCun}{2015b}]{Zbontar2015}
{\v{Z}}bontar, J. and LeCun, Y. (2015b).
\newblock Stereo matching by training a convolutional neural network to compare
  image patches.
\newblock {\em arXiv preprint arXiv:1510.05970}.

\bibitem[\protect\astroncite{{\v{Z}}bontar and LeCun}{2016}]{Zbontar2016}
{\v{Z}}bontar, J. and LeCun, Y. (2016).
\newblock Stereo matching by training a convolutional neural network to compare
  image patches.
\newblock {\em Journal of Machine Learning Research}, 17:1--32.

\bibitem[\protect\astroncite{Zheng et~al.}{2015}]{Zheng2015}
Zheng, S., Jayasumana, S., Romera-Paredes, B., Vineet, V., Su, Z., Du, D.,
  Huang, C., and Torr, P. H.~S. (2015).
\newblock Conditional random fields as recurrent neural networs.
\newblock In {\em International Conference on Computer Vision}.

\end{thebibliography}
